\documentclass[11pt]{article}
\usepackage[paper=letterpaper,margin=1in]{geometry}
\usepackage[numbers,sort,square]{natbib}

\title{Feature Cross Search via Submodular Optimization}

\author{Lin Chen\footnote{Authors are ordered alphabetically. Simons Institute for the Theory of Computing, University of California, Berkeley. E-mail: \texttt{lin.chen@berkeley.edu}.}
\and Hossein Esfandiari\footnote{Google Research, New York. E-mail: \texttt{esfandiari@google.com}.}
\and Gang Fu\footnote{Google Research, New York. E-mail: \texttt{thomasfu@google.com}.}
\and Vahab S. Mirrokni\footnote{Google Research, New York. E-mail: \texttt{mirrokni@google.com}.}
\and Qian Yu\footnote{Department of Electrical and Computer Engineering, University of Southern California. E-mail: \texttt{qyu880@usc.edu}.}
}

\usepackage[utf8]{inputenc} %
\usepackage[T1]{fontenc}    %
\usepackage[colorlinks=true]{hyperref}       %
\usepackage{amsfonts}       %

\usepackage{amsmath}
\usepackage{amsthm}
\usepackage{amssymb}

\usepackage{forloop}
\usepackage{xspace}

\usepackage{thmtools}
\usepackage{thm-restate}
\usepackage{mathtools}
\usepackage[title]{appendix}

\newtheorem{theorem}{Theorem}
\newtheorem{observation}[theorem]{Observation}
\newtheorem{lemma}[theorem]{Lemma}
\newtheorem{proposition}[theorem]{Proposition}
\newtheorem{corollary}[theorem]{Corollary}

\theoremstyle{definition}
\newtheorem{definition}{Definition}
\newtheorem{remark}{Remark}
\newtheorem{assumption}{Assumption}

\newcommand{\ie}{\emph{i.e.}}
\newcommand{\eg}{\emph{e.g.}}
\newcommand{\naive}{na\"ive Bayes}
\DeclareMathOperator{\auc}{AUC}

\DeclareMathOperator*{\argmax}{argmax}
\DeclareMathOperator{\poly}{\mathsf{poly}}
\DeclareMathOperator{\tpr}{TPR}
\DeclareMathOperator{\fpr}{FPR}
\DeclareMathOperator{\sign}{sign}

\newcommand{\indi}[1]{{\bf 1}\{#1\}}

\newcommand{\defcal}[1]{\expandafter\newcommand\csname 
	c#1\endcsname{{\mathcal{#1}}}}
\newcommand{\defbb}[1]{\expandafter\newcommand\csname 
	b#1\endcsname{{\mathbb{#1}}}}
\newcommand{\defbf}[1]{\expandafter\newcommand\csname 
	bf#1\endcsname{{\mathbf{#1}}}}
\newcounter{calBbCounter}
\forLoop{1}{26}{calBbCounter}{
	\edef\letter{\Alph{calBbCounter}}
	\expandafter\defcal\letter
	\expandafter\defbb\letter
	\expandafter\defbf\letter
}
\forLoop{1}{26}{calBbCounter}{
	\edef\letter{\alph{calBbCounter}}
	\expandafter\defbf\letter
}

\usepackage[capitalize]{cleveref}
\crefname{observation}{Observation}{Observations}
\crefname{assumption}{Assumption}{Assumptions}
\date{}
\begin{document}
	\maketitle
	
\begin{abstract}
In this paper, we study feature cross search as a fundamental primitive in feature engineering. The importance of feature cross search especially for the linear model has been known for a while, with well-known textbook examples. In this problem, the goal is to select a small subset of features, combine them to form a new feature (called the crossed feature) by considering their Cartesian product, and find feature crosses to learn an \emph{accurate} model. In particular, we study the problem of maximizing a normalized Area Under the Curve (AUC) of the linear model trained on the crossed feature column.

First, we show that it is not possible to provide an $n^{1/\log\log n}$-approximation algorithm for this problem unless the exponential time hypothesis fails. This result also rules out the possibility of solving this problem in polynomial time unless $\mathsf{P}=\mathsf{NP}$. On the positive side, by assuming the \naive\ assumption, we show that there exists a simple greedy $(1-1/e)$-approximation algorithm for this problem. This result is established by relating the AUC to the total variation of the commutator of two probability measures and showing that the total variation of the commutator is monotone and submodular. To show this, we relate the submodularity of this function to the positive semi-definiteness of a corresponding kernel matrix. Then, we use Bochner's theorem to prove the positive semi-definiteness by showing that its inverse Fourier transform is non-negative everywhere. Our techniques and structural results might be of independent interest.
\end{abstract}

\section{Introduction}
Feature engineering is one of the most fundamental problems in machine learning and it is the key to all supervised learning models. In feature engineering, we start with a collection of features (a.k.a., raw attributes) and turn them into a new set of features, with the purpose of improving the accuracy of the learning model. This is often done by some basic operations, such as removing irrelevant and redundant features (studied as feature selection~\cite{guyon2003introduction,weston2001feature,zadeh2017scalable,rogati2002high,hoque2014mifs,nie2010efficient,kwak2002input}), combining features (a.k.a., feature cross~\cite{WeiDGG19,LuoWZYTCDY19}) and bucketing and compressing the vocabulary of the features~\cite{bateni2019categorical,dhillon2003divisive,slonim2001power,bakeryz1998distributional}.

Finding an \emph{efficient} set of features to combine (\ie, cross) is one of the main primitives in feature engineering. Let us start with a text book example to show the importance of feature cross for the linear model. Consider a model with two features, language, which can be English or Spanish, and country, which can be Mexico or Scotland. Say if English appears with Scotland, or if Spanish appears with Mexico, the label is $1$. Otherwise the label is $0$. It is easy to see that in this case there is no linear model using these two features with a nontrivial accuracy (\ie, the best model matches the label with probability $1/2$). By crossing these two features, we get a new feature with four possible values (English, Mexico), (English, Scotland), (Spanish, Mexico), (Spanish, Scotland). Now, a linear model based on this new feature can perfectly match the label. This is a well-known concept in feature engineering.

Unlike feature selection and vocabulary compression, and despite the importance of feature cross search in practice, this problem is not well studied from a theoretical perspective. While some heuristics and exponential-time algorithms have been developed for this problem (\eg, ~\cite{WeiDGG19,LuoWZYTCDY19}), the complexity of designing approximation algorithms for this problem is not studied. This might be due to the complex behavior of crossing features on the accuracy of the learning models. In this work, we provide a simple formulation of this problem, and initiate a theoretical study.

Let us briefly define the problem as follows and defer the formal definition of the problem to a later section: Given a set of $n$ features, and a number $k$, compute a set of at most $k$ features out of $n$ features and combine these $k$ features such that the accuracy of the optimum linear model on the combined feature is maximized. To measure the accuracy we use normalized \emph{Area Under the Curve (AUC)}. The bound $k$ is to avoid over fitting.\footnote{In practice this number is chosen by tracking the accuracy of the model on the validation data. However, this is out of the scope of this paper.} This is a very basic definition for the feature cross search problem and can be considered as a building block in feature engineering. In fact, as we discuss later in the paper, it is still hard to design algorithms for this basic problem.

First, we show that there is no $n^{1/\log\log n}$-approximation algorithm for feature cross search unless the exponential time hypothesis fails. Our hardness result also implies that there does not exist a polynomial-time algorithm for feature cross search unless $\mathsf{P}=\mathsf{NP}$. It is easy to extend these hardness results to other notions of accuracy such as probability of matching the label. Obviously, this hardness result holds for any extension of the problem as well.

In fact, often, the real world inputs are not adversarially constructed. Usually, the inputs follow some structural properties that allow simple algorithms to work efficiently. With this intuition in mind, to complement our hardness result, for features under the \naive\  assumption~\cite{mitchell1997machine,turhan2009analysis,eyheramendy2003naive}, we provide a $(1-1/e)$-approximation algorithm that only needs  polynomially many function evaluations. We further discuss and justify this assumption in Section~\ref{sec:results}.

In \cref{sec:results}, we define the problem formally and present our results as well as an overview of our techniques.
In \cref{sec:prelim}, we provide the preliminary definitions and observations that will be used later in the proofs.
In \cref{sec:hard}, we present our hardness results.
We relate the maximum AUC to the log-likelihood ratio and the total variation of the commutator of two probability distributions in  \cref{sec:maximum-auc}. 
\cref{sec:commutator} establishes the monotonicty and submodularity of the maximum AUC as a set function. This section forms the most technical part of the paper.
In \cref{sec:related} we present other related works. Finally, \cref{sec:conclusion} concludes the paper.

\subsection{Problem Statement and Our Contributions}\label{sec:results}
We start with some definitions necessary to present our results, and then, we present our contributions.
Assume that the dataset comprises $n=|U|$ categorical feature columns and a binary label column, where $U$ is the set of all feature columns. Let the random variable $X_i$ denote the value of the $i$-th feature column ($i\in U$) and $C\in \{0,1\}$ be the value of the binary label. 
The random variables $ X_1,\dots,X_{|U|},C $ follow a joint 
distribution $ \cD $. Additionally, we assume that the support of the 
random variable $ X_i$ is a finite set $ V_i\subseteq \bN $. The set $V_i$ is also known as the \emph{vocabulary} of the $i$-th feature column. 
If
$ 
A\subseteq U $ is a set of feature columns, we write $ V_A $ for $ \prod_{a\in A} V_a  $
and write $ X_A $ for $ 
(X_a|a\in A) $,  where $(X_a|a\in A)$ is a vector indexed by $a\in A$ (for example, if $A=\{1,2,4\}$, the vector $X_A$ is a $3$-dimensional vector $(X_1,X_2,X_4)$.

Suppose that we focus on a set of feature columns and temporarily ignore the remaining feature columns. In other words, we consider the dataset modeled by the distribution of $(X_A, C)$. 
Let $\overline{\bR} = \bR\cup \{-\infty,+\infty\}$ denote the set of extended real numbers.
Given a function $\sigma:V_A\to \overline{\bR}$ that assigns a score to each possible value of $X_A$, and given a threshold $\tau$, we predict a positive label for $X_A$ if $\sigma(X_A)> \tau$ and predict a negative label if $\sigma(X_A)<\tau$. If $\sigma(X_A)=\tau$, we allow for predicting a positive or negative label at random. Let $\tpr$ and $\fpr$ denote the true and false positive rate of this model given a certain decision rule, respectively. Note that both true and false positive rates lie in $[0,1]$. %
If one varies $\tau$ from $-\infty$ to $\infty$ while fixing the score function $\sigma$, a curve that consists of the collection of achievable points $(\fpr,\tpr)$ is produced and the curve resides in the square $[0,1]\times [0,1]$. The area under the curve (AUC)~\cite{byrne2016note} is then defined as the area of the region enclosed by this curve, and the two lines $\fpr=1$ and $\tpr=0$. 

An equivalent definition is that AUC is roughly the probability that a random positive instance has a higher score (in terms of $\sigma$) than a negative instance (we say \emph{roughly} because in \cref{def:auc}, we have to be careful about tiebreaking, i.e., the second term). 
\begin{definition}[Area under the curve (AUC)~\cite{byrne2016note}]\label{def:auc}
	Given a set of feature columns $ A $ and a function  $ \sigma:V_A\to 
	\bR $, the area under the curve (AUC) of $ A $ and $ \sigma $ is
	\[
	    \auc_\sigma(A)
	    ={} \Pr[\sigma(X^+_A) > \sigma(X^-_A)|C^+=1,C^-=0] + 
	\frac{1}{2}\Pr[\sigma(X^+_A) = \sigma(X^-_A)|C^+=1,C^-=0]\,,
	\]
	where $ (X^+_A, C^+),(X^-_A,C^-)\sim \cD $ are i.i.d.\ and $X^\gamma_A = (X^\gamma_a|a\in A)$ obeys a marginal distribution of $\cD$ ($\gamma$ is either $+$ or $-$).
\end{definition}

The maximum AUC is the AUC of the best scoring function. 
It is a function of the set of feature columns and independent of the scoring function.

\begin{definition}[Maximum AUC]
    Given a set of feature columns $A$, the maximum AUC is 
    \[
        \auc^*(A) = \sup_{\sigma:V_A\to \bR} \auc_\sigma(A)\,.
    \]
\end{definition}

Now we are ready to present our results. We start with \cref{thm:auc-tv} which provides a characterization of AUC via the total variation distance. 

\begin{restatable}[AUC as total variation distance]{observation}{auctv}\label{thm:auc-tv}
	Let $ P^A_i $ be the conditional distribution $ \Pr[X_A|C=i] $ on $ V_A 
	$ and let $d_{TV}(P,Q)$ denote the total variation distance between two probability measures $P$ and $Q$.
	We have
	\[ \auc^*(A) ={} \frac{1}{2}+\frac{1}{2}d_{TV}(P^A_1\times P^A_0, P^A_0\times 
	P^A_1)
	={} \frac{1}{2} + \frac{1}{2} \sum_{\mathclap{x,y\in V_A}} \left| P_1^A(x)P_0^A(y)-P_0^A(x)P_1^A(y) \right| \,, \]
	where $P^A_1\times P^A_0$ and $ P^A_0\times 
	P^A_1$ denote the product measures. 
\end{restatable}

Recall that if $P$ and $Q$ are two probability measures on a common $\sigma$-algebra $\cF$, the \emph{total variation distance} between them is \[
d_{TV}(P,Q) \triangleq \sup_{A\in \cF} |P(A)-Q(A)|\in [0,1]\,.
\]
If the sample space $\Omega$ (the set of all outcomes) is finite, Scheff{\'e}'s lemma~\cite{scheffe1947useful} gives
\begin{equation}\label{eq:dtv-l1}
d_{TV}(P,Q) = \frac{1}{2}\|P-Q\|_1 \triangleq \frac{1}{2}\sum_{\omega\in \Omega}|P(\omega)-Q(\omega)|\,.
\end{equation}

    We present the proof of \cref{thm:auc-tv} in \cref{sec:maximum-auc}.
    \cref{thm:auc-tv} shows that the maximum AUC is an affine function of the total variation distance between $P^A_1\times P^A_0$ and $P^A_0\times P^A_1$, where $P^A_i$ is the probability measure conditioned on the label. In light of \eqref{eq:dtv-l1}, we have $d_{TV}(P^A_1\times P^A_0,P^A_0\times P^A_1) = \frac{1}{2}\|P^A_1\times P^A_0-P^A_0\times P^A_1\|_1$. We call the signed measure $P^A_1\times P^A_0-P^A_0\times P^A_1$ the \emph{commutator} of the two probability measures $P^A_0$ and $P^A_1$.
    Our second remark is that since the total variation distance always resides on $[0,1]$, the range of the maximum AUC is $[1/2,1]$.

The next theorem is our main hardness result, stating that it is not possible to approximate feature cross search, unless the exponential time hypothesis~\cite{impagliazzo2001complexity} fails. We consider maximization of $2\auc^*(A)-1$ rather than $\auc^*(A)$ in \eqref{eq:max-problem} because the range of the maximum AUC is $[1/2,1]$ (as we remark before) and assigning the same score to all feature values in $V_A$ attains an AUC of $1/2$, thereby achieving at least a $1/2$-approximation. In light of its range, we consider its normalized version $2\auc^*(A)-1$ whose range is $[0,1]$. 
We prove this theorem in Section~\ref{sec:hard}. In fact, our hardness result also implies that the feature cross search problem is NP-hard (see Corollary~\ref{cor:NPhard}).

\begin{restatable}{theorem}{hardmaxauc}\label{thm:hard-max-auc}
    There is no $ n^{1/\poly(\log \log n)} $-approximation algorithm for the following maximization problem unless the exponential time hypothesis~\cite{impagliazzo2001complexity} fails.
    \begin{equation}\label{eq:max-problem}
        \max_{A\subseteq U, |A|=k} (2\auc^*(A)-1)\,.
    \end{equation}
\end{restatable}

Although the above hardness result rules out the existence of an algorithm with a good approximation factor in the general case, it is very rare to face such hard examples in practice. We consider the \naive\ assumption that all feature columns are conditionally independent given the label. We borrowed this assumption from the widely-used na\"ive Bayes classifier~\cite{mitchell1997machine}. For example, under the same assumption, \cite{krause2005near} established the submodularity of mutual information and \cite{chen2015sequential} proved that in the sequential information maximization problem, the most
informative selection policy behaves near optimally.  \cite{turhan2009analysis} conducted an empirical study on the public software defect data from NASA with PCA pre-processing. They concluded that this assumption was not harmful. Although relaxing the assumption could produce numerically more favorable results, they were not statistically significantly better than assuming this assumption. In another example~\cite{eyheramendy2003naive}, based on their analysis on three real-world datasets for natural language processing tasks (MDR, Newsgroup and the ModApte version of the Reuters-21578), they drew a similar conclusion that relaxing the assumption did not improve the performance.

\begin{assumption}[Na\"ive Bayes]\label{asm:factorization}
	 Given the 
	label, all feature columns are independent. In other words, it holds for $ 
	A\subseteq U $ and $ i=0,1 $ that \begin{equation}\label{eq:factorization}
	    \Pr[X_A=x_A|C=i] = \prod_{a\in A} 
	\Pr[X_a=x_a|C=i]\,.
	\end{equation}
\end{assumption}

Our major algorithmic contribution is to show that under the \naive\ assumption the set function $\auc^*$ is monotone submodular, which in turn, implies that a greedy algorithm provides a constant-factor approximation algorithm for this problem.\footnote{We will review the definition of submodular and monotone set functions in \cref{sub:submodular-monotone}.} %

\begin{restatable}{theorem}{aucsubmodular}\label{thm:auc-submodular}
    Under the \naive\ assumption, the set function $\auc^*:2^U\to \bR$ is monotone submodular.
\end{restatable}

This theorem implies the following result in light of the result of \cite{nemhauser1978analysis}.%

\begin{corollary}
    Under the \naive\ assumption, there exists a $(1-1/e)$-approximation algorithm that only needs polynomially many  evaluations of $\auc^*$ for feature cross search.
\end{corollary}

To show \cref{thm:auc-submodular}, we prove \cref{thm:tv-submodular}. Proving this proposition requires an involved analysis and it may be of independent interest in statistics. 

\begin{restatable}[Proof in \cref{app:monotone,app:submodular}]{proposition}{tvsubmodular}\label{thm:tv-submodular}
    Let $U$ be a finite index set. Assume that for every $a\in U$, there are a pair of probability measures $P_0^a$ and $P_1^a$ on a common sample space $V_a$. 
	For any $ A\subseteq U $, define the set function $ F:2^{U}\to \bR_{\ge 
	0} $ by \begin{equation}\label{eq:fa-dtv}
	F(A) = d_{TV}\left(\bigtimes_{a\in A} P_1^a \times \bigtimes_{a\in A} 
	P_0^a, \bigtimes_{a\in A} P_0^a \times \bigtimes_{a\in A} 
	P_1^a \right)\,.
	\end{equation}
The set function $F$ is monotone and submodular.
\end{restatable}

Its proof is presented in \cref{sec:commutator}.
In fact, the most technical part of this paper is to prove \cref{thm:tv-submodular} which claims that the total variation of the commutator of probability measures is monotone submodular. The monotonicity is a consequence of the subadditivity of the absolute value function. Submodularity is the technically harder part and is shown in the following four steps. 
 
First, we introduce the notion of \emph{involution equivalence}. An involution is a map from a set to itself that is equal to its inverse map. Two probability measures $P$ and $P'$ are said to be involution equivalent if there exists an involution $f$ on the sample space $\Omega$ such that for every $x\in \Omega$, it holds that $P(x)=P'(f(x))$. 
Note that if $P(x)=P'(f(x))$ holds for every $x\in \Omega$, we have $P'(x)=P(f(x))$ also holds for every $x\in \Omega$. In fact, it defines a symmetric relation on probability measures on $\Omega$.  
If $P$ and $Q$ are two probability measures on a common sample space, the product measures $P\times Q$ and $Q\times P$ are involution equivalent and connected by the natural transpose involution $f$ that sends $(x,y)\in \Omega^2$ to $(y,x)\in \Omega^2$. 

The second step is in light of a key observation that summing a bivariate function of two involution equivalent probability measures over the common sample space remains invariant under the swapping of the two measures. Based on this key observation, if $P$ and $P'$ are involution equivalent, for every $x$ in their common sample space, we construct the probability measures of two Bernoulli random variables $U_x$ and $U'_x$ such that $U_x(1) = U'_x(0) = \frac{P(x)}{P(x)+P'(x)} $ and $U_x(0) = U'_x(1) =  \frac{P'(x)}{P(x)+P'(x)}$. 
The two Bernoulli probability measures $U_x$ and $U'_x$ are again involution equivalent and connected by the swapping of $0$ and $1$.
To establish submodularity, one has to check an inequality that characterizes the diminishing returns property (see equation~\eqref{eq:diminishing-returns} in \cref{sub:submodular-monotone}). Another key observation is that after defining the Bernoulli probability measures, the desired inequality can be shown to be a conic combination of the same inequality with \emph{some} (not all) probability measures in the inequality replaced by the Bernoulli probability measures $U_x$ and $U'_x$. 
To make the above observation work, we have to require that the remaining probability measures unreplaced in the inequality must be either of the form $P\times Q$ or its transpose $Q\times P$.
Using this approach, we reduce the problem to the Bernoulli case.

Third, after reducing the problem to the Bernoulli case, performing a series of more involved algebraic manipulations, we re-parametrize the desired inequality that formulates the diminishing returns property and obtain that the inequality is equivalent to the positive semi-definiteness of a quadratic form with respect to a kernel matrix.
However, this re-parametrization is valid only for elements of a positive measure with respect to some probability measures in the inequality. 
As a consequence, prior to the algebraic manipulations and re-parametrization, we have to eliminate those elements of measure zero by showing that their total contribution to the sum is zero. We would like to remark here that the individual terms may not be zero but they are canceled out under the summation. 

Finally, 
to show that the aforementioned kernel matrix is positive semi-definite, we prove that it is induced by a positive definite function. We establish the positive definiteness of the function by showing that its inverse Fourier transform is non-negative everywhere (this is an implication of the Bochner's theorem, see \cref{sec:prelim}).

\cref{thm:auc-submodular} is a straightforward corollary of \cref{thm:tv-submodular}. 
\begin{proof}
Let $P^A_i[\cdot]$ denote $\Pr[X_A|C=i]$, the conditional probability measure on $V_A$ given the labeling being $i$, where $i=0,1$. When $A=\{a\}$ is a singleton, we write $P_i^a$ for $P_i^{\{a\}}$ as a shorthand notation.
Under \cref{asm:factorization}, \eqref{eq:factorization} can be re-written as $
P_i^A[x_A] = \prod_{a\in A} P^a_i[x_a]$,
or in a more compact way, \[
P_i^A = \bigtimes_{a\in A} P_i^a\,.
\]
By \cref{thm:auc-tv}, we have \begin{align*}
     \auc^*(A)
    ={}& \frac{1}{2} + \frac{1}{2}d_{TV}(P_1^A\times P_0^A, P_0^A\times P_1^A) \\
    ={}& \frac{1}{2} + \frac{1}{2}d_{TV}(\bigtimes_{a\in A} P_1^a\times \bigtimes_{a\in A} P_0^a, \bigtimes_{a\in A} P_0^a\times \bigtimes_{a\in A} P_1^a)\\
    ={}& \frac{1}{2}+\frac{1}{2}F(A)\,.
\end{align*}
Since $F(A)$ is monotone submodular by \cref{thm:tv-submodular}, so is $\auc^*$.
\end{proof}

\section{Preliminaries}\label{sec:prelim}

Throughout this paper, let $\Delta_\Omega$ denote the set of all probability measures on a finite set $\Omega$ and we always assume that the sample space $\Omega$ is finite. The set of extended real numbers is denoted by $\overline{\bR}$ and defined as $\bR\cup \{-\infty,+\infty\}$. 

\subsection{Involution Equivalence}

We first review the definition of an involution. 

\begin{definition}[Involution]
	A map $f:\Omega \to \Omega$ is said to be an involution if for all $x\in \Omega$, it holds that $f(f(x))=x$.
\end{definition}

In this paper, we introduce a new notion termed \emph{involution equivalence}, which forms an equivalence relation on $\Delta_\Omega$. Intuitively, two probability measures on a common sample space $\Omega$ are involution equivalent if they are the same after renaming the elements in $\Omega$ via an involution.

\begin{definition}[Involution equivalence]
Let $P,P'$ be two probability measures on a common sample space $\Omega$. We say that $P$ and $P'$ are involution equivalent if there exists an involution $f$
such that for all $x\in \Omega$, $P(x)=P'(f(x))$. If $P$ and $P'$ are involution equivalent, we denote it by $P\stackrel{f}{\sim} P'$ or $P\sim P'$ with the involution $f$ omitted when it is not of our interest.
\end{definition}

\begin{remark}\label{rmk:ffx}
If $P\stackrel{f}{\sim} P'$, we have $P'(x)=P'(f(f(x))=P(f(x))$.
\end{remark}

\begin{remark}[Transpose involution]\label{rmk:transpose}
If $P$ and $P'$ are two probability measures on a common sample space $\Omega$, the product measure $P\times P'$ is involution equivalent to $P'\times P$ via the natural transpose map $\top$ that sends $(x,y)\in \Omega^2$ to $\top(x,y)=(y,x)\in \Omega^2$. Thus we write $P\times P'\stackrel{\top}{\sim}P'\times P$ and term $\top$ a \emph{transpose involution}.
\end{remark}

\subsection{Submodular and Monotone Set Functions}\label{sub:submodular-monotone}

Let us recall the definition of submodular and monotone set functions.
Submodular set functions are those satisfying that the marginal gain of adding a new element to a set is no smaller than that of adding the same element to its superset. This property is called  the \emph{diminishing returns property}, which naturally arises in data summarization~\cite{mitrovic2018data}, influence maximization~\cite{zhang2016influence}, and natural language processing~\cite{lin2012submodularity}, among others. 
\begin{definition}[Submodular set function, \cite{nemhauser1978analysis,krause2014submodular}]
    A set function $f:2^U\to \bR_{\ge 0}$ is \emph{submodular} if for any $A\subseteq U$ and $a,b\in \Omega\setminus A$ such that $a\ne b$, it satisfies \begin{equation}\label{eq:diminishing-returns}
        f(A\cup \{a\})- f(A) \ge f(A\cup \{a,b\}) - f(A\cup \{b\}) \,.
    \end{equation}
\end{definition}
The above Equation~\eqref{eq:diminishing-returns} formulates the diminishing returns property. Its left-hand side is the marginal gain of adding $a$ to a set $A$ while the right-hand side is the marginal gain of adding the same element $a$ to the superset $A\cup \{b\}$.

A monotone set function is a function that assigns a higher function value to a set than all its subsets.
\begin{definition}[Monotone set function]
    A set function $f:2^U\to \bR$ is \emph{monotone} if for any $A\subseteq B\subseteq U$, we have $f(A)\le f(B)$. 
\end{definition}

\section{Hardness Result}\label{sec:hard}

In this section we show the hardness of approximation of the feature cross search problem. We say an algorithm is an  $\alpha$-approximation algorithm for the feature cross search problem if its accuracy (\ie, $2 \auc -1$) is at least $\alpha$ times that of the optimum algorithm.

As a byproduct, we show a hardness result for a feature selection problem based on mutual information defined as follows.
In the label-based mutual information maximization problem we have a universe of features $  U $ and a vector of labels $ C $, and we want to select a subset S of size $ k $ from $ U $ that maximizes the mutual information $I(S;C)$. In other words we want to solve $ \argmax_{S\subseteq U,|S| = k}   I(S;C) $.  We say an algorithm is an $ \alpha $-approximation algorithm for the label-based mutual information maximization problem if it reports a set $ S $ such that  $ \alpha\le \frac{ I(S;C)}{\max_{S'\subseteq U , |S'| = k} I(S';C)} $.

For both problems, we show that an $ \alpha $-approximation algorithm for the problem implies an $ \alpha $-approximation algorithm for the $ k $-densest subgraph problem. In the $ k $-densest subgraph problem we are given a graph $ G(V,E) $ and a number $ k $ and we want to pick a subset $ S $ of size $ k $ from $ V $ such that the number of edges induced by $ S $ is maximized. Recently, \cite{manurangsi2017almost} shows that there is no [almost polynomial] $ n^{-1/\poly(\log \log n)} $- approximation algorithm for $ k $-densest subgraph that runs in polynomial time unless the exponential time hypothesis fails. The best known algorithm for this problem has approximation factor $ n^{-1/4} $ \cite{bhaskara2010detecting}.

    \hardmaxauc*

\begin{proof}
	We prove this theorem via an approximation preserving reduction from $ k 
	$-densest subgraph. Let $ G(V,E) $ be an instance of $ k $-densest subgraph 
	problem. We construct a set of features as follows. There are $ n=|V| $ 
	features each corresponding to one vertex of $ G $. For a vertex $ v\in V $ 
	we indicate the value of the feature corresponding to $ v $ by $ x_v $. 
	There are three possible feature values, $ 0 $, $ 1 $ and $ \# $. The 
	values of the features are determined by the following random process. 
	Select an edge $ (u,v) $ uniformly at random from $ E $. The value of the 
	features $ x_v $ and $ x_u $ are chosen independently and uniformly at 
	random from $ \{0,1\} $. The value of all other features are $ \# $. The 
	value of the label is $ x_v\oplus x_u $.  To show the hardness of 
	approximation  of the feature cross search, we show that any solution 
	of accuracy $ \phi =2 \auc -1 $ corresponds to a subgraph of $ G $ with $ k $ vertices 
	and $ \phi m $ edges and vise versa. %

	Let $ H $ be a subgraph of $ G $ with $ k $ vertices and $ \phi m $ edges. Let $ S $ be the set of features corresponding to the vertices in $ H $. We analyze this in two cases.
	
	\textbf{Case 1.} The value of the crossed feature contains zero or one numbers (\ie, all are $ \# $, or all but one are $ \# $). Note that this case corresponds to a scenario that the pair of features with binary value are not both in $S$ and hence it happens with probability $ \frac{m-\phi m}{m}=1-\phi $. %
	Moreover, note that in this case the value of the crossed feature is independent of the value of the label (i.e., given the value of the feature the label is 0 or 1 with probability $ 1/2 $ ). %
	
	\textbf{Case 2.} The value of the crossed feature contains two numbers. In this case one can easily predict the correct label with probability 1 (i.e., if the numbers are both 0 or both 1  output 0, otherwise output 1). Moreover, note that this case corresponds to a scenario that the pair of features with binary values are both in $S$ and hence it happen with probability $ \frac{\phi m}{m}=\phi $. 
	
	Case 1 happens with probability $ 1-\phi $ and in this case the label is independent of the value of the crossed feature, and Case 2 happens with probability $ \phi $, where the label can be predicted with probability 1. Therefore, we have $\auc= \int_0^1 \phi + (1-\phi) p dp = \phi + \frac {1-\phi} 2 = \frac{1+\phi} 2 $ which gives us $ 2\auc-1 = \phi $ as claimed.
\end{proof}

In fact, the densest subgraph problem in NP-hard as well, and hence the reduction in the proof of Theorem~\ref{thm:hard-max-auc} directly implies the NP-hardness of feature cross search as well. 

\begin{corollary}\label{cor:NPhard}
The feature cross search problem is NP-hard.
\end{corollary}

Similar proof to that of \cref{thm:hard-max-auc} implies the hardness of feature selection via label based mutual information maximization.

\begin{theorem}
\label{thm:hard-MI-feature-selection}
	There is no $ n^{-1/\poly(\log \log n)} $-approximation algorithm for feature selection via label based mutual information maximization unless the exponential time hypothesis fails.	
\end{theorem}
\begin{proof}
	Similar to \cref{thm:hard-max-auc} we prove this theorem via an approximation preserving reduction from $ k $-densest subgraph. Consider the hard example provided in the proof of \cref{thm:hard-max-auc}. Here we show that for any arbitrary set of features $ S $ if the induced subgraph of the corresponding vertices has $\phi m$ edges, we have $ I(S;C)=\phi $. We define a random variable $ X $ as follows. $ X $ is 0 if none or one of the features in $ S $ has a binary value, and $ X $ is 1 if two of the features in $ S $ have binary values. Note that the value of $ C $ is independent of $ X $, and thus we have $ I(S;C)= I(S;C|X) $. Hence, we have
	\begin{align*}
	    I(S;C)={}& I(S;C|X)
	    ={} \bE_X[ I(S,C) | X ]\\
	    ={}& \Pr[X=0] (I(S;C) | X=0) 
	     + \Pr[X=1] (I(S;C) | X=1)\,.
	\end{align*}
	Note that given $ X=0 $, $ S $ and $ C $ are independent and hence we have $ (I(S;C) | X=0)=0 $. On the other hand if $ X=1 $, $ S $ uniquely defines $ C $, and hence we have $ (I(S;C ) | X=1)=1 $. Therefore we have
	$ I(S;C)=\bE_X[ I(S,C) | X ]=\Pr[X=1] = \frac{\phi m}{m}=\phi $,
	as claimed. %
\end{proof}

\section{Reformulating Maximum AUC}\label{sec:maximum-auc}
Here, we prove \cref{thm:auc-tv} and thereby reformulate the maximum AUC as an affine function of the total variation of the commutator of two probability measures. Furthermore, we show that the maximum AUC is achieved by a specific scoring function, \ie, the log-likelihood ratio. 
	
We start with some definitions. We define the log-likelihood ratio of an event $ E $ by $ \cL(E) = \log 
\frac{P_1[E]}{P_0[E]} $ provided that $P_0[E]P_1[E]\ne 0$, where $ P_i[\cdot]=\Pr[\cdot|C=i] $. 
If $P_0[E]=0$, the log-likelihood ratio $\cL(E)$ is defined to be $+\infty$. If $P_1[E]=0$, it is defined to be $-\infty$. As a result, the range of the log-likelihood ratio is the set of extended real numbers, denoted by $\overline{\bR}= \bR \cup \{-\infty,+\infty\} $.
\cref{prop:max_auc} shows that the maximum AUC is achieved by a specific scoring function, \ie, the log-likelihood ratio  $\cL$. Here we abuse the notation and define the score $\cL(x_A)$ assigned to each $x_A\in V_A$ to be $\cL(X_A=x_A)$, where $X_A$ and $C$ are jointly sampled from $\mathcal{D}$.     
In other words, if we assign to each value in $V_A$ a score accordingly, then the AUC is maximized. 
	
	\begin{proposition}[Proof in \cref{app:max_auc}]\label{prop:max_auc} The log-likelihood ratio  
	 achieves the maximum AUC among all functions $ \sigma:V_A\to \overline{\bR} $
		\[
		\auc_{\cL}(A) = \max_{\sigma:V_A\to \overline{\bR}} \auc_\sigma(A)\,.
		\]
	\end{proposition}
	
	The above proposition is a folklore result. However, we provide a proof for completeness, and the proof steps are also used  to prove \cref{thm:auc-tv}.
	
			According to \cref{prop:max_auc} and equation (\ref{ineq:l_opt}), \cref{thm:auc-tv} directly follows.

\section{Total Variation of Commutator of Probability Measures}\label{sec:commutator}
In this section, we prove \cref{thm:tv-submodular}, which states that the total variation of commutator of probability measures is a monotone submodular set function. 

\tvsubmodular*

The monotonicity part is proved 
in \cref{app:monotone} and the submodularity part is proved in \cref{app:submodular}.

The submodularity of $F$ follows from \cref{lem:general-case}. 

\begin{restatable}[General case, proof in \cref{app:general-case}]{lemma}{generalcase}\label{lem:general-case}
	 Let $R,R'\in \Delta_{\Omega_1}$, $S,S'\in \Delta_{\Omega_2}$, and $P,Q\in \Delta_{\Omega}$, where $\Omega_1, \Omega_2,\Omega$ are finite sets.
	If $R\stackrel{f}{\sim} R'$ and $S\stackrel{g}{\sim} S'$,
	it holds that \begin{align*}
	 & d_{TV}(R\times S \times P \times Q, R'\times S'\times Q \times P)  -  d_{TV}(R\times P \times Q, R'\times Q \times P)\\
	 -{} &{} d_{TV}(S \times P\times Q, S'\times Q \times P)
	 +  d_{TV}(P\times Q, Q\times P) \le 0\,.
	\end{align*}
\end{restatable}

We begin with the Bernoulli case where $\Omega_1$ and $\Omega_2$ in its statement are both $\{0,1\}$ so that $R,R',S,S'$ are all probability measures of a Bernoulli random variable. 

\begin{restatable}[Bernoulli case, proof in \cref{app:bernoulli}]{lemma}{bernoullicase}\label{lem:bernoulli}
    Let $R,R',S,S'\in \Delta_{\{0,1\}}$ such that $R\stackrel{f}{\sim}  R'$ and $S\stackrel{f}{\sim} S'$, where $f$ is a function on $\{0,1\}$ such that $f(0)=1$ and $f(1)=0$. Let $P,Q\in \Delta_\Omega$, where $\Omega$ is a finite sample space. The following inequality holds 
        \begin{equation}\label{eq:bernoulli-submodular-inequality}
        \begin{split}
          &  d_{TV}(R\times S \times P \times Q, R'\times S'\times Q \times P) 
          -{}  d_{TV}(R\times P \times Q, R'\times Q \times P)\\
	 -{} & d_{TV}(S \times P\times Q, S'\times Q \times P)
	 +{}  d_{TV}(P\times Q, Q\times P) \le 0\,.
        \end{split}
        \end{equation}
\end{restatable}

\begin{proof}[Proof sketch]
    To prove the Bernoulli case, we first show that under the summation (recall that according to Equation~\eqref{eq:dtv-l1}, the total variation distance is half of the $L^1$ distance, and the $L^1$ distance is the sum of the absolute value of the difference on each singleton), any term that involves an element of measure zero (with respect to $P$ or $Q$) has no contribution to the expression on the left-hand side. We would like to emphasize that while the term itself may be non-zero, it will be canceled out under the summation. In our second step, we will consider quantities of the form $\sqrt{\frac{P(x)Q(y)}{Q(x)P(y)}}$ in which $Q(x)$ and $P(y)$ must be non-zero for all $x$ and $y$. As a result, we have to eliminate elements of measure zero in first step by showing that their total contribution is zero. 
    
    As the second step, we perform a series of algebraic manipulations and substitutions and finally show that the opposite of left-hand side can be re-written as a quadratic $v^\top Mv$, where $v$ is a vector and $M$ is a symmetric square matrix. Recall that the promised inequality claims that the left-hand side is non-positive (thus the opposite of the left-hand side is non-negative). Therefore, we will show it by establishing the positive semi-definiteness of $M$. 
    
    In fact, the matrix $M$ is induced by a positive definite function. The problem of establishing the positive semi-definiteness of $M$ reduces to the problem of proving that the function that induces $M$ is positive definite. In light of the Bochner's theorem (see \cref{lem:bochner} in \cref{sec:bochner}), we show its positive definiteness by computing its inverse Fourier transform, which turns out to be finite-valued and non-negative everywhere. 
\end{proof}

The high-level strategy of proving \cref{lem:general-case} is to use \cref{obs:general-sum} to reduce the problem to the Bernoulli case (\cref{lem:bernoulli}). The proof details can be found in \cref{app:general-case}. 
    
    \begin{observation}
    \label{obs:general-sum}
        Let $P,P'\in \Delta_\Omega$ be such that $P\stackrel{f}{\sim}P'$ and $\phi:\bR^2\to \bR$ be a homogeneous bivariate function, i.e., $\phi(\lambda x,\lambda y)=\lambda \phi(x,y)$ holds for any $x,y,\lambda\in \bR$. 
        For every element $x\in \Omega$, we define the Bernoulli probability measure $U_x$ on $\{0,1\}$ such that $U_x(1)=\frac{P(x)}{P(x)+P'(x)}$ and $U'_x(1)=\frac{P'(x)}{P(x)+P'(x)}$.
        The following equation holds 
        \begin{align*}
             \sum_{x\in \Omega}\phi(P(x),P'(x))
            ={} \sum_{x\in \Omega} \frac{P(x)+P'(x)}{2} \left(\phi(U_x(1),U'_x(1)) + \phi(U'_x(1),U_x(1))\right)
            \,.
        \end{align*}
    \end{observation}
\begin{proof}
    Under the assumption of the observation statement, we have
    \begin{align*}
            \sum_{x\in \Omega}\phi(P(x),P'(x))
            ={}& \frac{1}{2}\sum_{x\in \Omega}\phi(P(x),P'(x)) + \frac{1}{2}\sum_{x\in \Omega}\phi(P(x),P'(x))\\
            ={}& \frac{1}{2}\sum_{x\in \Omega}\phi(P(x),P'(x)) + \frac{1}{2}\sum_{x\in \Omega}\phi(P'(x),P(x))\\
            ={}& \sum_{x\in \Omega} \frac{P(x)+P'(x)}{2} \left(\phi(U_x(1),U'_x(1)) + \phi(U'_x(1),U_x(1))\right)
            \,.
        \end{align*}
        We use \cref{lem:swap} in the second term on the second line and the third equality is because $\phi$ is homogeneous.
    \end{proof}

\section{Other Related Works}\label{sec:related}
As discussed before, our problem falls in the category of feature engineering problems.
Perhaps, the most studied problem in feature engineering is feature selection~\cite{guyon2003introduction, weston2001feature,zadeh2017scalable,rogati2002high,hoque2014mifs,nie2010efficient,kwak2002input,wei2015submodularity}. In this problem, the goal is to select a small subset of the features to obtain a learning model with high accuracy and avoid over-fitting.
Here we just mention a couple of feature selection algorithm related to submodular maximization and refer to \cite{guyon2003introduction} for an introduction to feature selection and many relevant references.
\cite{elenberg2016restricted} used the notion of weak submodularity to design and analyze feature selection algorithms. \cite{krause2006near} used the submodularity of mutual information between the sensors to design a $(1-1/e)$-approximation algorithm for sensor placements, which can be directly used for feature selection. However, as we show in \cref{thm:hard-MI-feature-selection}, it is not possible to design such algorithms to maximize the mutual information between the features and the label. 

Another related well-studied problem in this domain is vocabulary compression~\cite{bateni2019categorical,dhillon2003divisive,slonim2001power,bakeryz1998distributional}. The goal of vocabulary compression is to improve the learning and serving time, and in some cases to avoid overfitting. Vocabulary compression can be done by simple approaches such as filtering and naive bucketing, or more complex approaches such as mutual information maximization. \cite{bakeryz1998distributional} and \cite{slonim2001power} used clustering algorithms based on the Jenson-Shannon divergence to compress the vocabulary of features. \cite{dhillon2003divisive} proposed an iterative algorithm that locally maximizes the mutual information between a feature and the label. Recently, \cite{bateni2019categorical} considered this problem for binary labels and presented a quasi-linear-time distributed approximation algorithm to maximize the mutual information between the feature and the label. There are polynomial-time local algorithms for binary labels that maximize the mutual information~\cite{kurkoski2014quantization,iwata2014quantizer}, studied in the context of discrete memoryless channels.

\cite{WeiDGG19} designed an integer programming based algorithm for feature cross search and applied it to learn generalized linear models using rule-based features. They show that this approach obtains better accuracy compared to that of the existing rule ensemble algorithms. \cite{LuoWZYTCDY19} proposed a greedy algorithm for feature cross search and show that the greedy algorithm works well on a variety of datasets. Neither of these papers provide any theoretical guarantees for the performance of their algorithm.

\section{Conclusion}\label{sec:conclusion}
In this paper, we considered the problem of feature cross search. We formulated it as a problem of maximizing the normalized area under the curve (AUC) of the linear model trained on the crossed feature column. We first established a hardness result that no algorithm can provide $n^{1/\log\log n}$ approximation for this problem unless the exponential time hypothesis fails. Therefore, no polynomial algorithm can solve this problem unless $\mathsf{P}=\mathsf{NP}$. In light of its intractable nature, we motivated and assumed the \naive\ assumption. We related AUC to the total variation of the commutator of two probability measures. Under the \naive\ assumption, we demonstrated that the aforementioned total variation is monotone and submodular with respect to the set of selected feature columns to be crossed. As a result, a greedy algorithm can achieve a $(1-1/e)$-approximation of the problem. Our proof techniques may be of independent interest. Finally, an empirical study showed that the greedy algorithm outperformed the baselines.

\bibliography{reference-list}

\begin{thebibliography}{34}
\providecommand{\natexlab}[1]{#1}
\providecommand{\url}[1]{\texttt{#1}}
\expandafter\ifx\csname urlstyle\endcsname\relax
  \providecommand{\doi}[1]{doi: #1}\else
  \providecommand{\doi}{doi: \begingroup \urlstyle{rm}\Url}\fi

\bibitem[Baker and McCallum(1998)]{bakeryz1998distributional}
L~Douglas Baker and Andrew~Kachites McCallum.
\newblock Distributional clustering of words for text classification.
\newblock In \emph{Proceedings of the 21st annual international ACM SIGIR
  conference on Research and development in information retrieval}, pages
  96--103. ACM, 1998.

\bibitem[Bateni et~al.(2019)Bateni, Chen, Esfandiari, Fu, Mirrokni, and
  Rostamizadeh]{bateni2019categorical}
Mohammadhossein Bateni, Lin Chen, Hossein Esfandiari, Thomas Fu, Vahab
  Mirrokni, and Afshin Rostamizadeh.
\newblock Categorical feature compression via submodular optimization.
\newblock In \emph{International Conference on Machine Learning}, pages
  515--523, 2019.

\bibitem[Bhaskara et~al.(2010)Bhaskara, Charikar, Chlamtac, Feige, and
  Vijayaraghavan]{bhaskara2010detecting}
Aditya Bhaskara, Moses Charikar, Eden Chlamtac, Uriel Feige, and Aravindan
  Vijayaraghavan.
\newblock Detecting high log-densities: an $o (n^{1/4})$ approximation for
  densest k-subgraph.
\newblock In \emph{STOC}, pages 201--210. ACM, 2010.

\bibitem[Byrne(2016)]{byrne2016note}
Simon Byrne.
\newblock A note on the use of empirical auc for evaluating probabilistic
  forecasts.
\newblock \emph{Electronic Journal of Statistics}, 10\penalty0 (1):\penalty0
  380--393, 2016.

\bibitem[Chen et~al.(2015)Chen, Hassani, Karbasi, and
  Krause]{chen2015sequential}
Yuxin Chen, S~Hamed Hassani, Amin Karbasi, and Andreas Krause.
\newblock Sequential information maximization: When is greedy near-optimal?
\newblock In \emph{Conference on Learning Theory}, pages 338--363, 2015.

\bibitem[Cheney and Light(2009)]{cheney2009course}
Elliott~Ward Cheney and William~Allan Light.
\newblock \emph{A course in approximation theory}, volume 101.
\newblock American Mathematical Soc., 2009.

\bibitem[Dhillon et~al.(2003)Dhillon, Mallela, and Kumar]{dhillon2003divisive}
Inderjit~S Dhillon, Subramanyam Mallela, and Rahul Kumar.
\newblock A divisive information-theoretic feature clustering algorithm for
  text classification.
\newblock \emph{Journal of machine learning research}, 3\penalty0
  (Mar):\penalty0 1265--1287, 2003.

\bibitem[Elenberg et~al.(2018)Elenberg, Khanna, Dimakis, and
  Negahban]{elenberg2016restricted}
Ethan~R Elenberg, Rajiv Khanna, Alexandros~G Dimakis, and Sahand Negahban.
\newblock Restricted strong convexity implies weak submodularity.
\newblock \emph{The Annals of Statistics}, 46\penalty0 (6B):\penalty0
  3539--3568, 2018.

\bibitem[Eyheramendy et~al.(2003)Eyheramendy, Lewis, and
  Madigan]{eyheramendy2003naive}
Susana Eyheramendy, David~D Lewis, and David Madigan.
\newblock On the naive bayes model for text categorization.
\newblock In \emph{9th International Workshop on Artificial Intelligence and
  Statistics}. Citeseer, 2003.

\bibitem[Guyon and Elisseeff(2003)]{guyon2003introduction}
Isabelle Guyon and Andr{\'e} Elisseeff.
\newblock An introduction to variable and feature selection.
\newblock \emph{Journal of machine learning research}, 3\penalty0
  (Mar):\penalty0 1157--1182, 2003.

\bibitem[Hoque et~al.(2014)Hoque, Bhattacharyya, and Kalita]{hoque2014mifs}
Nazrul Hoque, Dhruba~K Bhattacharyya, and Jugal~K Kalita.
\newblock Mifs-nd: A mutual information-based feature selection method.
\newblock \emph{Expert Systems with Applications}, 41\penalty0 (14):\penalty0
  6371--6385, 2014.

\bibitem[Impagliazzo and Paturi(2001)]{impagliazzo2001complexity}
Russell Impagliazzo and Ramamohan Paturi.
\newblock On the complexity of k-sat.
\newblock \emph{Journal of Computer and System Sciences}, 62\penalty0
  (2):\penalty0 367--375, 2001.

\bibitem[Iwata and Ozawa(2014)]{iwata2014quantizer}
Ken-ichi Iwata and Shin-ya Ozawa.
\newblock Quantizer design for outputs of binary-input discrete memoryless
  channels using smawk algorithm.
\newblock In \emph{2014 IEEE International Symposium on Information Theory},
  pages 191--195. IEEE, 2014.

\bibitem[Krause and Golovin(2014)]{krause2014submodular}
Andreas Krause and Daniel Golovin.
\newblock Submodular function maximization.
\newblock In \emph{Tractability: Practical Approaches to Hard Problems}, pages
  71--104. Cambridge University Press, 2014.

\bibitem[Krause and Guestrin(2005)]{krause2005near}
Andreas Krause and Carlos Guestrin.
\newblock Near-optimal nonmyopic value of information in graphical models.
\newblock In \emph{Proceedings of the Twenty-First Conference on Uncertainty in
  Artificial Intelligence}, pages 324--331. AUAI Press, 2005.

\bibitem[Krause et~al.(2006)Krause, Guestrin, Gupta, and
  Kleinberg]{krause2006near}
Andreas Krause, Carlos Guestrin, Anupam Gupta, and Jon Kleinberg.
\newblock Near-optimal sensor placements: Maximizing information while
  minimizing communication cost.
\newblock In \emph{Proceedings of the 5th international conference on
  Information processing in sensor networks}, pages 2--10. ACM, 2006.

\bibitem[Kurkoski and Yagi(2014)]{kurkoski2014quantization}
Brian~M Kurkoski and Hideki Yagi.
\newblock Quantization of binary-input discrete memoryless channels.
\newblock \emph{IEEE Transactions on Information Theory}, 60\penalty0
  (8):\penalty0 4544--4552, 2014.

\bibitem[Kwak and Choi(2002)]{kwak2002input}
Nojun Kwak and Chong-Ho Choi.
\newblock Input feature selection by mutual information based on parzen window.
\newblock \emph{IEEE transactions on pattern analysis and machine
  intelligence}, 24\penalty0 (12):\penalty0 1667--1671, 2002.

\bibitem[Lin(2012)]{lin2012submodularity}
Hui Lin.
\newblock \emph{Submodularity in natural language processing: algorithms and
  applications}.
\newblock PhD thesis, University of Washington, 2012.

\bibitem[Luo et~al.(2019)Luo, Wang, Zhou, Yao, Tu, Chen, Dai, and
  Yang]{LuoWZYTCDY19}
Yuanfei Luo, Mengshuo Wang, Hao Zhou, Quanming Yao, Wei{-}Wei Tu, Yuqiang Chen,
  Wenyuan Dai, and Qiang Yang.
\newblock Autocross: Automatic feature crossing for tabular data in real-world
  applications.
\newblock In \emph{Proceedings of the 25th {ACM} {SIGKDD} International
  Conference on Knowledge Discovery {\&} Data Mining, {KDD} 2019, Anchorage,
  AK, USA, August 4-8, 2019.}, pages 1936--1945, 2019.

\bibitem[Manurangsi(2017)]{manurangsi2017almost}
Pasin Manurangsi.
\newblock Almost-polynomial ratio eth-hardness of approximating densest
  k-subgraph.
\newblock In \emph{STOC}, pages 954--961. ACM, 2017.

\bibitem[Mitchell(1997)]{mitchell1997machine}
Tom Mitchell.
\newblock \emph{Machine Learning}.
\newblock McGraw-Hill International Editions. McGraw-Hill, 1997.

\bibitem[Mitrovic et~al.(2018)Mitrovic, Kazemi, Zadimoghaddam, and
  Karbasi]{mitrovic2018data}
Marko Mitrovic, Ehsan Kazemi, Morteza Zadimoghaddam, and Amin Karbasi.
\newblock Data summarization at scale: A two-stage submodular approach.
\newblock In \emph{ICML}, pages 3593--3602, 2018.

\bibitem[Nemhauser et~al.(1978)Nemhauser, Wolsey, and
  Fisher]{nemhauser1978analysis}
George~L Nemhauser, Laurence~A Wolsey, and Marshall~L Fisher.
\newblock An analysis of approximations for maximizing submodular set
  functions—i.
\newblock \emph{Mathematical programming}, 14\penalty0 (1):\penalty0 265--294,
  1978.

\bibitem[Nie et~al.(2010)Nie, Huang, Cai, and Ding]{nie2010efficient}
Feiping Nie, Heng Huang, Xiao Cai, and Chris~H Ding.
\newblock Efficient and robust feature selection via joint $\ell_{2,1}$-norms
  minimization.
\newblock In \emph{Advances in neural information processing systems}, pages
  1813--1821, 2010.

\bibitem[Rogati and Yang(2002)]{rogati2002high}
Monica Rogati and Yiming Yang.
\newblock High-performing feature selection for text classification.
\newblock In \emph{Proceedings of the eleventh international conference on
  Information and knowledge management}, pages 659--661. ACM, 2002.

\bibitem[Scheff{\'e}(1947)]{scheffe1947useful}
Henry Scheff{\'e}.
\newblock A useful convergence theorem for probability distributions.
\newblock \emph{The Annals of Mathematical Statistics}, 18\penalty0
  (3):\penalty0 434--438, 1947.

\bibitem[Slonim and Tishby(2001)]{slonim2001power}
Noam Slonim and Naftali Tishby.
\newblock The power of word clusters for text classification.
\newblock In \emph{23rd European Colloquium on Information Retrieval Research},
  volume~1, page 200, 2001.

\bibitem[Turhan and Bener(2009)]{turhan2009analysis}
Burak Turhan and Ayse Bener.
\newblock Analysis of naive bayes’ assumptions on software fault data: An
  empirical study.
\newblock \emph{Data \& Knowledge Engineering}, 68\penalty0 (2):\penalty0
  278--290, 2009.

\bibitem[Wei et~al.(2019)Wei, Dash, Gao, and G{\"{u}}nl{\"{u}}k]{WeiDGG19}
Dennis Wei, Sanjeeb Dash, Tian Gao, and Oktay G{\"{u}}nl{\"{u}}k.
\newblock Generalized linear rule models.
\newblock In \emph{Proceedings of the 36th International Conference on Machine
  Learning, {ICML} 2019, 9-15 June 2019, Long Beach, California, {USA}}, pages
  6687--6696, 2019.

\bibitem[Wei et~al.(2015)Wei, Iyer, and Bilmes]{wei2015submodularity}
Kai Wei, Rishabh Iyer, and Jeff Bilmes.
\newblock Submodularity in data subset selection and active learning.
\newblock In \emph{International Conference on Machine Learning}, pages
  1954--1963. PMLR, 2015.

\bibitem[Weston et~al.(2001)Weston, Mukherjee, Chapelle, Pontil, Poggio, and
  Vapnik]{weston2001feature}
Jason Weston, Sayan Mukherjee, Olivier Chapelle, Massimiliano Pontil, Tomaso
  Poggio, and Vladimir Vapnik.
\newblock Feature selection for svms.
\newblock In \emph{Advances in neural information processing systems}, pages
  668--674, 2001.

\bibitem[Zadeh et~al.(2017)Zadeh, Ghadiri, Mirrokni, and
  Zadimoghaddam]{zadeh2017scalable}
Sepehr~Abbasi Zadeh, Mehrdad Ghadiri, Vahab Mirrokni, and Morteza
  Zadimoghaddam.
\newblock Scalable feature selection via distributed diversity maximization.
\newblock In \emph{Thirty-First AAAI Conference on Artificial Intelligence},
  2017.

\bibitem[Zhang et~al.(2016)Zhang, Bai, Chen, Bian, and Li]{zhang2016influence}
Yuanxing Zhang, Yichong Bai, Lin Chen, Kaigui Bian, and Xiaoming Li.
\newblock Influence maximization in messenger-based social networks.
\newblock In \emph{GLOBECOM}, pages 1--6. IEEE, 2016.

\end{thebibliography}

\newpage
\begin{appendices}
\crefalias{section}{appendix}
\crefalias{subsection}{appendix}

\section{Positive Definite Function and Bochner's Theorem}\label{sec:bochner}

We review the definition of a positive definite function and the Bochner's theorem that provides the equivalent characterization of positive definite functions. They are indeed the Fourier transform of a non-negative finite-valued Borel measure.

\begin{definition}[Positive definite function]
    A complex-valued function $f:\bR\to \bC$ is positive definite if for any real numbers $x_1,x_2,\dots,x_n$, the $n\times n$ matrix $A$ whose $(i,j)$ entry is $f(x_i-x_j)$ is positive semi-definite.
\end{definition}

\begin{lemma}[Bochner, \cite{cheney2009course}]\label{lem:bochner}
    A function $f:\bR\to \bC$ is positive definite and continuous if and only if it is the Fourier transform of a non-negative finite-valued function on $\bR$. 
\end{lemma}

Additionally, we review the definition of the Fourier transform and inverse Fourier transform.
\begin{definition}[Fourier transform]
    The Fourier transform of a function $f(t)$ is \[
        (\cF f)(x) = \frac{1}{\sqrt{2\pi}} \int_{-\infty}^{\infty} f(t) e^{ixt}dt\,.
    \]
\end{definition}

\begin{definition}[Inverse Fourier transform]
    The inverse Fourier transform of the function $F(t)$ is \[
        (\cF^{-1} F)(x) = \frac{1}{\sqrt{2\pi}} \int_{-\infty}^{\infty} F(t) e^{-itx} dt\,.
    \]
\end{definition}

The following lemma shows that a shift results in an additional multiplicative factor in the inverse Fourier transform. This property can be established using integration by substitution.
\begin{lemma}[Inverse Fourier transform of a shifted function]\label{lem:shifted-ift}
    If the inverse Fourier transform of $F(t)$ is $f(x)$, then the inverse Fourier transform of $F(t+a)$ is $e^{iax}f(x)$.
\end{lemma}
\begin{proof}
    Let us compute the inverse Fourier transform of $F(t+a)$
    \[
        \frac{1}{\sqrt{2\pi}}\int_{-\infty}^{\infty} F(t+a)e^{-itx}dt 
        = \frac{e^{iax}}{\sqrt{2\pi}}\int_{-\infty}^{\infty} F(t+a)e^{-i(t+a)x}d(t+a) 
        = e^{iax} f(x)\,.
    \]
\end{proof}

\section{Proof of \cref{prop:max_auc} and \cref{thm:auc-tv}}\label{app:max_auc}
\begin{proof}%
		To prove the above proposition, it suffices to show that for any scoring function $\sigma$, its achieved AUC is no greater than that achieved by using the log-likelihood ratio as the scoring function. In other words, we aim to prove that for any $\sigma:V_A\to \overline{\bR}$,
				\[
		\auc_{\cL}(A) \geq \auc_{\sigma}(A).
		\]
		
			Recall \cref{def:auc}. The AUC given a scoring function $\sigma$ can be described using  i.i.d.\ random variables $ (X^+_U, C^+),(X^-_U,C^-)\sim \cD $. We can express this quantity using indicator functions as follows
				\begin{align*}
				\auc_\sigma(A)  &=\bE[\indi{\sigma(X^+_A) > \sigma(X^-_A)}+\frac{1}{2}\indi{\sigma(X^+_A) = \sigma(X^-_A)}|C^+=1,C^-=0]\,.
				\end{align*}
				Note that $\indi{\sigma(X^+_A) > \sigma(X^-_A)}+\indi{\sigma(X^+_A) = \sigma(X^-_A)}+\indi{\sigma(X^+_A) < \sigma(X^-_A)}=1$, we have
				\begin{equation}\label{eq:auc_exp}
					\begin{split}
				\auc_\sigma(A)  &=\frac{1}{2}+\bE\left[\frac{1}{2}\indi{\sigma(X^+_A) > \sigma(X^-_A)}-\frac{1}{2}\indi{\sigma(X^+_A) < \sigma(X^-_A)}|C^+=1,C^-=0\right]\,\\
				&=\frac{1}{2}+\frac{1}{2}\bE\left[\indi{\sigma(X^+_A) > \sigma(X^-_A)}-\indi{\sigma(X^+_A) < \sigma(X^-_A)}|C^+=1,C^-=0\right]\,\\
				&=\frac{1}{2}+\frac{1}{2}\bE\left[\sign{(\sigma(X^+_A) - \sigma(X^-_A))}|C^+=1,C^-=0\right]\,.
				\end{split}
				\end{equation}

				To maximize the AUC, we focus on the term $\bE[\sign{(\sigma(X^+_A) - \sigma(X^-_A))}|C^+=1,C^-=0]$. 
	By symmetrizing this quantity, we obtain the following equations.
	\begin{equation}
				\begin{split}
			& \bE[\sign{(\sigma(X^+_A) - \sigma(X^-_A))}|C^+=1,C^-=0]\\
			={}&\frac{1}{2}
			\left(\bE[\sign{(\sigma(X^+_A) - \sigma(X^-_A))}|C^+=1,C^-=0]+\bE[\sign{(\sigma(X^-_A) - \sigma(X^+_A))}|C^+=0,C^-=1] \right)\,\\
				={}&\frac{1}{2}
			\left(\bE[\sign{(\sigma(X^+_A) - \sigma(X^-_A))}|C^+=1,C^-=0]-\bE[\sign{(\sigma(X^+_A) - \sigma(X^-_A))}|C^+=0,C^-=1] \right)\,\\
				={}&\frac{1}{2} \sum_{x^+_A,x^-_A\in V_A} \left(P_1[x^+_A]P_0[x^-_A]\sign{(\sigma(x^+_A) - \sigma(x^-_A))}- P_1[x^-_A]P_0[x^+_A]\sign{(\sigma(x^+_A) - \sigma(x^-_A))}\right)\,.\\
					={}&\frac{1}{2} \sum_{x^+_A,x^-_A\in V_A} \left(P_1[x^+_A]P_0[x^-_A]- P_1[x^-_A]P_0[x^+_A]\right)\sign{(\sigma(x^+_A) - \sigma(x^-_A))}\,.
				\end{split}
				\end{equation}
			The above expression can be upper bounded by the total variation distance between $P_1\times P_0$ and $P_0\times P_1$.
				\begin{equation}\label{ineq:tv_bound}
				\begin{split}
			\bE[\sign{(\sigma(X^+_A) - \sigma(X^-_A))}|C^+=1,C^-=0]
			&\leq \frac{1}{2} \sum_{x^+_A,x^-_A\in V_A} \left|P_1[x^+_A]P_0[x^-_A]- P_1[x^-_A]P_0[x^+_A]\right|\,\\
			&=d_{TV}(P_1^A\times P_0^A,P_0^A\times P_1^A) \,.
				\end{split}
				\end{equation}
				Note that whenever $P_1[x^+_A]P_0[x^-_A]$ or $P_1[x^-_A]P_0[x^+_A]$ is non-zero, the log-likelihood ratios $\cL(x_A^+)$ and $\cL(x_A^-)$, as well as $\cL(x_A^+)$-$\cL(x_A^-)$, are well defined on $\overline{\bR}$. Hence, if $P_1[x^+_A]P_0[x^-_A]-P_1[x^-_A]P_0[x^+_A]\neq 0$, one can show that 
				\[
				\sign(P_1[x^+_A]P_0[x^-_A]-P_1[x^-_A]P_0[x^+_A])=\sign{(\cL(x^+_A) - \cL(x^-_A))}.
				\]
			Consequently, all equality conditions in (\ref{ineq:tv_bound}) can be achieved by using log-likelihood ratio as the scoring function, and we have 
				\begin{equation}\label{ineq:interm_l_opt}
				\begin{split}
			\bE[\sign{(\sigma(X^+_A) - \sigma(X^-_A))}|C^+=1,C^-=0]
			&\leq \bE[\sign{(\cL(X^+_A) - \cL(X^-_A))}|C^+=1,C^-=0]\,\\
			&=d_{TV}(P_1^A\times P_0^A,P_0^A\times P_1^A) \,.
				\end{split}
				\end{equation}
				Combining (\ref{eq:auc_exp}) and (\ref{ineq:interm_l_opt}), we have the following bound that holds true for any $\sigma$,
					\begin{equation}\label{ineq:l_opt}
				\begin{split}
			\auc_{\sigma}(A)
			\leq\auc_{\cL}(A)
			&=\frac{1}{2} +\frac{1}{2} d_{TV}(P_1^A\times P_0^A,P_0^A\times P_1^A) \,,
				\end{split}
				\end{equation}
				which completes the proof.
			\end{proof}

	\section{Proof of \cref{lem:bernoulli}}\label{app:bernoulli}
	\begin{proof}
    We will use the following shorthand notations
    \[
        r={} R(1)=R'(0)\,,    r'={} R(0)=R'(1)\,,  s={} S(1)=S'(0)\,,  s'={} S(0)=S'(1)\,.\\
    \]
    The above four equations hold because $R\stackrel{f}{\sim}  R'$ and $S\stackrel{f}{\sim} S'$.
    Using the above notation and expanding the summation for $R,R',S,S'$, we have  
    \begin{equation}\label{eq:tv-rspq}
         \begin{split}
        & 2d_{TV}(R\times S \times P \times Q, R'\times S'\times Q \times P)\\
        ={} & \sum_{z,w} (|R(1)S(1) P(z)Q(w)-R'(1)S'(1) Q(z)P(w)| + |R(1)S(0) P(z)Q(w)-R'(1)S'(0) Q(z)P(w)| \\
       & + |R(0)S(1) P(z)Q(w) - R'(0)S'(1) Q(z)P(w)|
         + |R(0)S(0) P(z)Q(w) - R'(0)S'(0) Q(z)P(w)|)\\
        ={} & \sum_{z,w} (|rs P(z)Q(w)-r's' Q(z)P(w)| + |rs' P(z)Q(w)-r's Q(z)P(w)| \\
       & + |r's P(z)Q(w) - rs' Q(z)P(w)| + |r's' P(z)Q(w) - rs Q(z)P(w)|\,.
    \end{split}
    \end{equation}
    Performing a similar algebraic manipulation on $d_{TV}(R\times P \times Q, R'\times Q \times P) $, $ d_{TV}(S \times P\times Q, S'\times Q \times P) $ and $ d_{TV}(P\times Q, Q\times P)$ yields 
    \begin{equation}\label{eq:tv-rpq}
        \begin{split}
            & 2d_{TV}(R\times P \times Q, R'\times Q \times P) \\
        ={} & \sum_{z,w} |R(1) P(z)Q(w)-R'(1) Q(z)P(w)| + |R(0) P(z)Q(w)-R'(0) Q(z)P(w)|\\
        ={} & \sum_{z,w}|r P(z)Q(w)-r' Q(z)P(w)| + |r' P(z)Q(w)-r Q(z)P(w)|\,,
        \end{split}
    \end{equation}
    and
    \begin{equation}\label{eq:tv-spq}
        \begin{split}
            & 2d_{TV}(S \times P\times Q, S'\times Q \times P) \\
        ={}& \sum_{z,w}|S(1) P(z)Q(w) - S'(1) Q(z)P(w)| + |S(0) P(z)Q(w) - S'(0) Q(z)P(w)| \\
        ={}& \sum_{z,w}|s P(z)Q(w) - s' Q(z)P(w)| + |s' P(z)Q(w) - s Q(z)P(w)|\,.
        \end{split}
    \end{equation}
    
    Plugging \eqref{eq:tv-rspq}, \eqref{eq:tv-rpq} and \eqref{eq:tv-spq} into the left-hand side of \eqref{eq:bernoulli-submodular-inequality}, we get
    \begin{equation}\label{eq:e}
        \begin{split}
       & \frac{1}{2}\sum_{z,w}  ( |rs P(z)Q(w)-r's' Q(z)P(w)| + |rs' P(z)Q(w)-r's Q(z)P(w)| \\
       & + |r's P(z)Q(w) - rs' Q(z)P(w)|
         + |r's' P(z)Q(w) - rs Q(z)P(w)|\\
         & - |r P(z)Q(w)-r' Q(z)P(w)| - |r' P(z)Q(w)-r Q(z)P(w)|\\
    &- |s P(z)Q(w) - s' Q(z)P(w)| - |s' P(z)Q(w) - s Q(z)P(w)|\\
    & + |P(z)Q(w) - Q(z)P(w)|)
         \\
    ={} & \sum_{z,w} (  
    |rs P(z)Q(w)-r's' Q(z)P(w)| + |rs' P(z)Q(w)-r's Q(z)P(w)| \\
   & - |r P(z)Q(w)-r' Q(z)P(w)|
     - |s P(z)Q(w) - s' Q(z)P(w)|\\
     & + \frac{1}{2}|P(z)Q(w) - Q(z)P(w)|
    )\,.
    \end{split}
    \end{equation}
    The above equality is because the following four pairs of terms are indeed equal under summation by renaming $z$ to $w$ and renaming $w$ to $z$
    \begin{align*}
        \sum_{z,w} |rs P(z)Q(w)-r's' Q(z)P(w)| ={}& \sum_{z,w} |r's' P(z)Q(w) - rs Q(z)P(w)| \\  \sum_{z,w} |rs' P(z)Q(w)-r's Q(z)P(w)| ={}& \sum_{z,w} |r's P(z)Q(w) - rs' Q(z)P(w)|\\ \sum_{z,w} |r P(z)Q(w)-r' Q(z)P(w)| ={}& \sum_{z,w} |r' P(z)Q(w)-r Q(z)P(w)|\\
        \sum_{z,w} |s P(z)Q(w) - s' Q(z)P(w)| ={}& \sum_{z,w} |s' P(z)Q(w) - s Q(z)P(w)|\,.
    \end{align*}
    
    We define $E$ to be the negation of the summand in \eqref{eq:e}:
    \begin{align*}
        E(z,w) ={} & \frac{1}{2}  (|r P(z)Q(w)-r' Q(z)P(w)| + |r' P(z)Q(w)-r Q(z)P(w)|\\
    &+ |s P(z)Q(w) - s' Q(z)P(w)| + |s' P(z)Q(w) - s Q(z)P(w)|\\
       & -|rs P(z)Q(w)-r's' Q(z)P(w)| - |rs' P(z)Q(w)-r's Q(z)P(w)| \\
       & - |r's P(z)Q(w) - rs' Q(z)P(w)|
         - |r's' P(z)Q(w) - rs Q(z)P(w)|\\
    & - |P(z)Q(w) - Q(z)P(w)|)\\
    ={}& |r P(z)Q(w)-r' Q(z)P(w)|
     + |s P(z)Q(w) - s' Q(z)P(w)|\\
    & -|rs P(z)Q(w)-r's' Q(z)P(w)| - |rs' P(z)Q(w)-r's Q(z)P(w)| \\
     & - \frac{1}{2}|P(z)Q(w) - Q(z)P(w)|
    \end{align*}
 and $D = \sum_{z,w} E(z,w)$, we have $D$ is the negated left-hand side of \eqref{eq:bernoulli-submodular-inequality}. 
 To show that the left-hand side of \eqref{eq:bernoulli-submodular-inequality} is non-positive, it suffices to prove that $D$ is non-negative.
 
    First we show that in the summation $\sum_{z,w} E(z,w)$, the contribution of $z$ and $w$ satisfying $P(z)P(w)=0$ is zero. Let $Z_P=\{ i\in \Omega | P(i)=0 \}$ denote the set of all elements in $\Omega$ that have measure zero with respect to $P$. We compute the total contribution of $z$ and $w$ if exactly one of them has measure zero with respect to $P$. We have
    \begin{equation*}
    \sum_{(z,w)\in Z_P\times \Omega} E(z,w) = \sum_{(z,w)\in Z_P\times \Omega}(r'+s'-r's'-r's-\frac{1}{2})Q(z)P(w)
    = \sum_{(z,w)\in Z_P\times \Omega}(s'-\frac{1}{2})Q(z)P(w)  
    \end{equation*}
    and 
    \begin{equation}\label{eq:w-zp}
    \begin{split}
        \sum_{(z,w)\in \Omega\times Z_P} E(z,w) ={} & \sum_{(z,w)\in \Omega\times Z_P}(r+s-rs-rs'-\frac{1}{2})P(z)Q(w) = \sum_{(z,w)\in \Omega\times Z_P}(s-\frac{1}{2}) P(z)Q(w)\\
    ={} & 
    \sum_{(z,w)\in Z_P\times \Omega}(s-\frac{1}{2}) Q(z)P(w)
    \,,
    \end{split}
    \end{equation}
    where the last equality in \eqref{eq:w-zp} is obtained by renaming $z$ to $w$ and $w$ to $z$.
    The set of pairs $\{(z,w)\in \Omega\times \Omega| P(z)P(w)=0\}$ can be expressed as the union of $Z_P\times \Omega$ and $\Omega\times Z_P$, and their intersection is $Z_P\times Z_P$. 
    Since $\sum_{(z,w)\in Z_P\times Z_P} E(z,w) = 0$, we get
    \begin{align*}
        \sum_{z,w: P(z)P(w)=0} E(z,w) ={} & \sum_{(z,w)\in Z_P\times \Omega} E(z,w) + \sum_{(z,w)\in \Omega\times Z_P } E(z,w) - \sum_{(z,w)\in Z_P\times Z_P} E(z,w)\\
        ={} & \sum_{(z,w)\in Z_P\times \Omega} E(z,w) + \sum_{(z,w)\in \Omega\times Z_P } E(z,w) \\
        ={} & \sum_{(z,w)\in Z_P\times \Omega} (s'+s-1)Q(z)P(w) = 0\,.
    \end{align*}

    Next, we consider $Z_P^c=\Omega \setminus Z_P$ that excludes $Z_P$ from $\Omega$, and $Z'_Q=\{ i\in Z_P^c| Q(i)=0 \}$, the elements in $Z_P^c$ that have measure zero with respect to $Q$. We compute the total contribution of pairs $(z,w)$ in $Z_P^c\times Z_P^c$ such that exactly one of them has measure zero with respect to $Q$. We have 
    \begin{equation*}
        \sum_{(z,w)\in Z'_Q\times Z_P^c} E(z,w) = \sum_{(z,w)\in Z'_Q\times Z_P^c}
        (r+s-rs-rs'-\frac{1}{2})P(z)Q(w) = \sum_{(z,w)\in Z'_Q\times Z_P^c} (s-\frac{1}{2}) P(z)Q(w)
    \end{equation*}
    and
    \begin{equation}\label{eq:w-zq}
    \begin{split}
        \sum_{(z,w)\in Z_P^c\times Z'_Q} E(z,w) ={} & \sum_{(z,w)\in Z_P^c\times Z'_Q} (r'+s'-r's'-r's-\frac{1}{2})Q(z)P(w)\\
        ={}& \sum_{(z,w)\in Z_P^c\times Z'_Q} (s'-\frac{1}{2})Q(z)P(w)\\
        ={} & \sum_{(z,w)\in Z'_Q\times Z_P^c} (s'-\frac{1}{2})P(z)Q(w)\,,
    \end{split}
    \end{equation}
    where the last equality in \eqref{eq:w-zq} is obtained by renaming $z$ to $w$ and $w$ to $z$. 
    The set of pairs $\{(z,w)\in \Omega\times \Omega| P(z)P(w)\ne 0, Q(z)Q(w)=0\}$ is the union of $Z'_Q\times Z_P^c$ and $Z_P^c\times Z'_Q$, and their intersection is $Z'_Q\times Z'_Q$. 
    Since $\sum_{(z,w)\in Z'_Q\times Z'_Q} E(z,w) = 0$, we obtain
    \begin{align*}
        & \sum_{z,w:P(z)P(w)\ne 0, Q(z)Q(w)=0} E(z,w)\\
        ={}& \sum_{(z,w)\in Z'_Q\times Z_P^c} E(z,w) + \sum_{(z,w)\in Z_P^c\times Z'_Q} E(z,w) - \sum_{(z,w)\in Z'_Q\times Z'_Q} E(z,w) \\
        ={} & \sum_{(z,w)\in Z'_Q\times Z_P^c} E(z,w) + \sum_{(z,w)\in Z_P^c\times Z'_Q} E(z,w) \\
        ={} &  \sum_{(z,w)\in Z'_Q\times Z_P^c} (s+s'-1) P(z)Q(w) = 0\,.
    \end{align*}
    Therefore, we conclude that
    \[
        \sum_{z,w:P(z)P(w)Q(z)Q(w)=0} E(z,w) = 0\,.
    \]
    
    In the sequel, we focus on $z,w\in \Omega_+$, where $\Omega_+ = \{i\in \Omega| P(i)Q(i)\ne 0 \}$. 
    Since the contribution of the pairs $(z,w)$ outside $\Omega^+\times \Omega^+$ turns out to be zero, we deduce
    \begin{align*}
        D ={}& \sum_{(z,w)\in \Omega_+\times \Omega_+} E(z,w) \\
        ={}& \frac{1}{2}\sum_{(z,w)\in \Omega_+\times \Omega_+} [ |r P(z)Q(w)-r' Q(z)P(w)|
        + |r' P(z)Q(w)-r Q(z)P(w)|\\
    & + |s P(z)Q(w) - s' Q(z)P(w)|
    +|s' P(z)Q(w) - s Q(z)P(w)|\\
   & -|rs P(z)Q(w)-r's' Q(z)P(w)|
   -|r's' P(z)Q(w)-rs Q(z)P(w)|\\
   &- |rs' P(z)Q(w)-r's Q(z)P(w)| 
   - |r's P(z)Q(w)-rs' Q(z)P(w)|\\
   &   - |P(z)Q(w) - Q(z)P(w)|]\,.
    \end{align*}
    
    Since all elements in $\Omega^+$ that we consider have a positive measure with respect to $P$ and $Q$, we define $l(z) = \frac{1}{2}\log\frac{P(z)}{Q(z)}$ for all $z\in \Omega^+$. Additionally, we define $v(z)=\sqrt{P(z)Q(z)}$ and $d(z,w) = e^{l(z)-l(w)}$.
    Notice that $d(z,w) = \sqrt{\frac{P(z)Q(w)}{Q(z)P(w)}}$. With these notations, we can re-write $D$ as 
    \begin{align*}
        D  ={}& \frac{1}{2}\sum_{(z,w)\in \Omega_+\times \Omega_+} v(z)v(w) 
   [ |r d(z,w)-r'/d(z,w)|
        + |r' d(z,w) -r / d(z,w)|\\
    & + |s d(z,w) - s' / d(z,w)|
    +|s' d(z,w) - s / d(z,w)|\\
   & -|rs d(z,w)-r's' / d(z,w)|
   -|r's' d(z,w)-rs / d(z,w)|\\
   &- |rs' d(z,w)-r's / d(z,w)| 
   - |r's d(z,w)-rs' / d(z,w)|\\
   &   - |d(z,w) -1 / d(z,w)|]\,.
    \end{align*}
    
  If we further define the function $c(t) = |rt-r'/t|+|r't-r/t|-|t-1/t| $, $D$ can be expressed in a more compact way
    \begin{align}
      D  ={} & \frac{1}{2}\sum_{(z,w)\in \Omega_+\times \Omega_+} v(z)v(w)
   \left[c(d(z,w))-\sqrt{ss'}(c(\sqrt{\frac{s}{s'}}d(z,w)) +c(\sqrt{\frac{s'}{s}}d(z,w)) \right]\nonumber\\
   \triangleq{} & \frac{1}{2}\sum_{(z,w)\in \Omega_+\times \Omega_+} v(z)v(w) M(z,w)\,, \label{eq:matrix-form}
    \end{align}
    where we let $M(z,w)$ denote $c(d(z,w))-\sqrt{ss'}(c(\sqrt{\frac{s}{s'}}d(z,w)) +c(\sqrt{\frac{s'}{s}}d(z,w)) $. In \eqref{eq:matrix-form}, we obtain a linear algebraic form of $D$, in which $M$ is a $|\Omega_+|\times |\Omega_+|$ matrix and $v$ is a $|\Omega_+|$-dimensional vector. Our next step is to show that $M$ is a symmetric matrix, \ie, 
     $M(z,w)=M(w,z)$.
    Since $c(t)=c(1/t)$ and $d(w,z)=1/d(z,w)$, we have
    \begin{align*}
        M(w,z) ={} & c(d(w,z))-\sqrt{ss'}\left[
        c(\sqrt{\frac{s}{s'}}d(w,z))+c(\sqrt{\frac{s'}{s}}d(w,z))
        \right]\\
        ={}& c(1/d(z,w))-\sqrt{ss'}\left[
        c(\sqrt{\frac{s}{s'}}/d(z,w))+c(\sqrt{\frac{s'}{s}}/d(z,w))
        \right]\\
        ={}& c(d(z,w))-\sqrt{ss'}\left[
        c(\sqrt{\frac{s'}{s}}d(z,w))+c(\sqrt{\frac{s}{s'}}/d(z,w))
        \right]\\
        ={}& M(z,w)\,.
    \end{align*}
    Recall that our goal is to show that $D\ge 0$. We will achieve this goal by establishing that $M$ is a $|\Omega_+|\times |\Omega_+|$  positive semi-definite matrix. 
    
    If we define \[
    m(t) = c(e^{t}) - \sqrt{ss'}\left[
        c(\sqrt{\frac{s}{s'}}e^{t})
        +c(\sqrt{\frac{s'}{s}}e^{t})
        \right]\,,
    \]
    we can re-write $M(z,w)$ as
    \begin{equation*}
        M(z,w) = c(e^{l(z)-l(w)}) - \sqrt{ss'}\left[
        c(\sqrt{\frac{s}{s'}}e^{l(z)-l(w)})
        +c(\sqrt{\frac{s'}{s}}e^{l(z)-l(w)})
        \right] = m(l(z)-l(w))\,.
    \end{equation*}
    The matrix $M$ is positive semi-definite provided that $m$ is proved to be a positive definite function. In the sequel, we establish the positive semi-definiteness of $M$ by proving that $m$ is a positive definite function. By Bochner's theorem (\cref{lem:bochner}), it is sufficient to show that the inverse Fourier transform of $m$ is non-negative. 
    
    As the first step of computing the inverse Fourier transform of $m$, we compute the inverse Fourier transform of $c(e^t)$. Without loss of generality, we assume $r\ge r'$. Notice that if $|t|\ge \frac{1}{2}\log \frac{r}{r'}$, we have $c(e^t)=0$.
    As a result, the inverse Fourier transform of $c(e^t)$ is 
    \begin{align*}
        \tilde{c}(x) ={} & \frac{1}{\sqrt{2\pi}}\int_{|t|\le \frac{1}{2}\log\frac{r}{r'}} c(e^t)e^{-itx}dt\\
        ={}& \frac{2}{\sqrt{2\pi}}\left[\int_0^{\frac{1}{2}\log\frac{r}{r'}} [ r e^{-t}-r' e^{t} ] e^{-itx}dt 
        + \int_{-\frac{1}{2}\log\frac{r}{r'}}^0 [ r e^{t}-r' e^{-t} ] e^{-itx}dt \right]\\
        ={}& \frac{4 \left(r+r' \left(-e^{(\frac{1}{2}+\frac{i 
        x}{2})\log\frac{r}{r'}}-e^{(\frac{1}{2}-\frac{i
         x}{2})\log \frac{r}{r'}}+1\right)\right)}{\sqrt{2\pi}(x^2+1)}\\
         ={}& \frac{4}{\sqrt{2\pi}(1+x^2)} \left(1-2\sqrt{rr'}\cos\left(\frac{1}{2}x\log\frac{r}{r'}\right)\right)
        \,,
    \end{align*}
    where the last equality comes from the fact that $r+r'=1$.
    By \cref{lem:shifted-ift}, the inverse Fourier transform of $c(\sqrt{\frac{s}{s'}}e^t) = c(e^{t+\frac{1}{2}\log \frac{s}{s'}})$ is 
    $
    e^{\frac{1}{2}i\log \frac{s}{s'}x}\tilde{c}(x)
    $. 
    Similarly, the inverse Fourier transform of $c(\sqrt{\frac{s'}{s}}e^t) = c(e^{t+\frac{1}{2}\log \frac{s'}{s}})$ is 
    $
    e^{\frac{1}{2}i\log \frac{s'}{s}x}\tilde{c}(x)
    $.
    Thus the inverse Fourier transform of $m(t)$ is 
    \begin{align*}
        \tilde{m}(x) ={} & (1-\sqrt{ss'}(e^{\frac{1}{2}ix\log\frac{s}{s'}}+e^{\frac{1}{2}ix\log\frac{s'}{s}}))\tilde{c}(x) \\
        ={}& (1-2\sqrt{ss'}\cos\left(\frac{1}{2}x\log\frac{s}{s'}\right))\tilde{c}(x)\\
        ={} & \frac{4}{1+x^2}\left(1-2\sqrt{ss'}\cos\left(\frac{1}{2}x\log\frac{s}{s'}\right)\right)
        \left(1-2\sqrt{rr'}\cos\left(\frac{1}{2}x\log\frac{r}{r'}\right)\right)\\
        \ge{}& 0\,,
    \end{align*}
    where the inequality holds because 
    \begin{align*}
        2\sqrt{ss'}\cos\left(\frac{1}{2}x\log\frac{s}{s'}\right) \le{} & 2\sqrt{ss'} \le s+s'=1\,,\\
        2\sqrt{rr'}\cos\left(\frac{1}{2}x\log\frac{r}{r'}\right) \le{} & 2\sqrt{rr'} \le r+r'=1\,.
    \end{align*}

    \end{proof}
\section{Proof of \cref{lem:general-case}}
\subsection{Involutionary Swapping Lemma}
Before presenting the proofs, we introduce the involutionary swapping lemma. 
Intuitively, the involutionary swapping lemma implies that two involution equivalent probability measures can be swapped inside a summation of a bivariate function. 
	
\begin{lemma}[Involutionary swapping lemma]\label{lem:swap}
    Let $P,P'\in \Delta_\Omega$ be such that $P\stackrel{f}{\sim}P'$ and $\phi:\bR^2\to \bR$ be any bivariate function. Then we have \[
        \sum_{x\in \Omega} \phi(P(x),P'(x)) = \sum_{x\in \Omega} \phi(P'(x),P(x))\,.
    \]
\end{lemma}
\begin{proof}
        Under the assumption of the lemma statement, we have \begin{align*}
            & \sum_{x\in \Omega} \phi(P(x),P'(x)) \\
            ={} & \sum_{x\in \Omega} \phi(P'(f(x)),P(f(x))) \\
            ={} & \sum_{x'\in \Omega} \phi(P'(x'),P(x')) \\
            ={} & \sum_{x\in \Omega} \phi(P'(x),P(x))\,.
        \end{align*}
        The first equality is because for any $x\in \Omega$, we have $P(x)=P'(f(x))$ (by the definition of involution equivalence) and $P(f(x))=P'(x)$ (\cref{rmk:ffx}). The second equality is obtained by setting $x'=f(x)$ (this is because any involution map $f$ is a bijection). The final equality is obtained by renaming $x'$ to $x$.
    \end{proof}
	\subsection{Proof of \cref{lem:general-case}}\label{app:general-case}
\begin{proof}%
	For every element $x\in \Omega_1$, we define two Bernoulli probability measures $U_x$ and $U'_x$ on $\{0,1\}$ such that  $U_x(1)=\frac{R(x)}{R(x)+R'(x)}$,  $U'_x(1)=\frac{R'(x)}{R(x)+R'(x)}$, $U_x(0)=1-U_x(1)$, and $U'_x(0)=1-U'_x(1)$. Similarly for every element $y\in \Omega_2$, we define two Bernoulli probability measures $V$ and $V'$ on $\{0,1\}$ such that $V_y(1)=\frac{S(y)}{S(y)+S'(y)}$,  $V'_y(1)=\frac{S'(y)}{S(y)+S'(y)}$, $V_y(0)=1-V_y(1)$, and $V'_y(0)=1-V'_y(1)$. 
	Note that 
	$U_x\sim U'_x$ and $V_y\sim V'_y$ via the involution map that swaps $0$ and $1$.
	
	Recalling the definition of the total variation distance, we have \begin{equation}\label{eq:dtv-first}
	    d_{TV}(R\times S\times P\times Q, R'\times S'\times Q \times P) 
	= \frac{1}{2}\sum_{x,y,z,w} \left| R(x)S(y)P(z)Q(w)-R'(x)S'(y)Q(z)P(w) \right|\,.
	\end{equation}
	Since  $S\stackrel{g}{\sim} S'$, \cref{lem:swap} implies 
	\begin{equation}\label{eq:dtv-swap}
	    \sum_{x,y}|R(x)S(y)P(z)Q(w)-R'(x)S'(y)Q(z)P(w)|
	   = \sum_{x,y}|R(x)S'(y)P(z)Q(w)-R'(x)S(y)Q(z)P(w)|\,. 
	\end{equation}
	Combining \eqref{eq:dtv-first} and \eqref{eq:dtv-swap} gives 
	\begin{align*}
	& d_{TV}(R\times S\times P\times Q, R'\times S'\times Q \times P) \\
	={} & \frac{1}{2}\sum_{x,y,z,w} \left| R(x)S(y)P(z)Q(w)-R'(x)S'(y)Q(z)P(w) \right| \\
	={}& \frac{1}{4}\sum_{x,y,z,w}\left( |R(x)S(y)P(z)Q(w)-R'(x)S'(y)Q(z)P(w)| \right.\\
	& \left.+ |R(x)S(y)P(z)Q(w)-R'(x)S'(y)Q(z)P(w)| \right)\\
	={} & \frac{1}{4}\sum_{x,y,z,w}\left( |R(x)S(y)P(z)Q(w)-R'(x)S'(y)Q(z)P(w)|\right.\\
	& \left.+ |R(x)S'(y)P(z)Q(w)-R'(x)S(y)Q(z)P(w)| \right)\\
	={} & \sum_{x,y,z,w}\frac{(R(x)+R'(x))(S(y)+S'(y))}{4} ( |U_x(1)V_y(1)P(z)Q(w)-U'_x(1)V'_y(1)Q(z)P(w)| \\ &+|U_x(1)V'_y(1)P(z)Q(w)-U'_x(1)V_y(1)Q(z)P(w)| )\\
	={}& \sum_{x,y}\frac{(R(x)+R'(x))(S(y)+S'(y))}{4}\sum_{z,w} ( |U_x(1)V_y(1)P(z)Q(w)-U'_x(1)V'_y(1)Q(z)P(w)| \\ &+|U_x(1)V_y(0)P(z)Q(w)-U'_x(1)V'_y(0)Q(z)P(w)| )\,,
	\end{align*}
	where we use \eqref{eq:dtv-swap} in the third equality.

    The inside summation turns out to be a total variation distance
	\begin{align*}
	  &\sum_{z,w}  (|U_x(1)V_y(1)P(z)Q(w)-U'_x(1)V'_y(1)Q(z)P(w)| \\  &+|U_x(1)V_y(0)P(z)Q(w)-U'_x(1)V'_y(0)Q(z)P(w)|) \\
	  ={}&  \frac{1}{2}\sum_{z,w}
	 (|U_x(1)V_y(1)P(z)Q(w)-U'_x(1)V'_y(1)Q(z)P(w)| \\
	 &+|U'_x(1)V'_y(1)P(z)Q(w)-U_x(1)V_y(1)Q(z)P(w)|\\
	 &+|U_x(1)V_y(0)P(z)Q(w)-U'_x(1)V'_y(0)Q(z)P(w)| \\
	 &+|U'_x(1)V'_y(0)P(z)Q(w)-U_x(1)V_y(0)Q(z)P(w)|)\\
	 ={} & \frac{1}{2}\sum_{z,w}
	 (|U_x(1)V_y(1)P(z)Q(w)-U'_x(1)V'_y(1)Q(z)P(w)| \\
	 &+|U_x(0)V_y(0)P(z)Q(w)-U'_x(0)V'_y(0)Q(z)P(w)|\\
	 &+|U_x(1)V_y(0)P(z)Q(w)-U'_x(1)V'_y(0)Q(z)P(w)| \\
	 &+|U_x(0)V_y(1)P(z)Q(w)-U'_x(0)V'_y(1)Q(z)P(w)|)\\
	={}&  d_{TV}(U_x\times V_y \times P \times Q, U'_x\times V'_y \times Q \times P)\,.
	\end{align*}
	Therefore, we get \begin{align*}
	   & d_{TV}(R\times S\times P\times Q, R'\times S'\times Q \times P) \\
	    ={}& \sum_{x,y}\frac{(R(x)+R'(x))(S(y)+S'(y))}{4}d_{TV}(U_x\times V_y \times P \times Q, U'_x\times V'_y \times Q \times P)\,.
	\end{align*}
	
	Since $R\stackrel{f}{\sim} R'$, we have 
	\begin{align*}
	  &  d_{TV}(R\times P \times Q, R'\times Q \times P) \\
	  ={} & \frac{1}{2}\sum_{x,z,w} |R(x)P(z)Q(w)-R'(x)Q(z)P(w)|\\
	  ={} & \frac{1}{4}\sum_{x,z,w} (|R(x)P(z)Q(w)-R'(x)Q(z)P(w)| + |R(x)P(z)Q(w)-R'(x)Q(z)P(w)|)\\
	  ={} & \frac{1}{4}\sum_{x,z,w} (|R(x)P(z)Q(w)-R'(x)Q(z)P(w)| + |R'(x)P(z)Q(w)-R(x)Q(z)P(w)|)\\
	  ={} & \frac{1}{4}
	  \sum_{x,z,w}(R(x)+R'(x))(|U_x(1)P(z)Q(w)-U'_x(1)Q(z)P(w)|+
	  |U'_x(1)P(z)Q(w)-U_x(1)Q(z)P(w)|)\\
	 ={}& \frac{1}{4}
	  \sum_{x}(R(x)+R'(x)) \sum_{z,w} (|U_x(1)P(z)Q(w)-U'_x(1)Q(z)P(w)|\\
	  & +
	  |U_x(0)P(z)Q(w)-U'_x(0)Q(z)P(w)|)
	  \,,
	\end{align*}
	where we use \cref{lem:swap} in the third equality. The inside summation is again a total variation distance \begin{align*}
	   &  \sum_{z,w} (|U_x(1)P(z)Q(w)-U'_x(1)Q(z)P(w)|+
	  |U_x(0)P(z)Q(w)-U'_x(0)Q(z)P(w)|)\\
	  ={}& 2 d_{TV}(U_x\times P\times Q, U'_x\times Q \times P) \,.
	\end{align*}
	Therefore we obtain \begin{align*}
	   & d_{TV}(R\times P \times Q, R'\times Q \times P) \\
	    ={}& \frac{1}{2}
	  \sum_{x}(R(x)+R'(x)) d_{TV}(U_x\times P\times Q, U'_x\times Q \times P)\\ 
	  ={}& \sum_{x,y} \frac{(R(x)+R'(x))(S(y)+S'(y))}{4} d_{TV}(U_x\times P\times Q, U'_x\times Q \times P)\,.
	\end{align*}

	Similarly, by $S\sim S'$, we have 
	\begin{align*}
	    d_{TV}(S \times P\times Q, S'\times Q \times P)
	  = \sum_{x,y} \frac{(R(x)+R'(x))(S(y)+S'(y))}{4}
	  d_{TV}(V_y\times P\times Q, V'_y \times Q \times P)\,.
	\end{align*}
	
	Therefore, the left-hand side of the inequality that we would like to show can be re-written as 
	\begin{align*}
	    & \sum_{x,y} \frac{(R(x)+R'(x))(S(y)+S'(y))}{4}(
	    d_{TV}(U_x\times V_y \times P \times Q, U'_x\times V'_y \times Q \times P)\\
	    & - d_{TV}(U_x\times P\times Q, U'_x\times Q \times P)
	    - d_{TV}(V_y\times P\times Q, V'_y \times Q \times P)
	    + d_{TV}(P\times Q, Q\times P)
	    ) \le 0\,.
	\end{align*}
        By \cref{lem:bernoulli} for the Bernoulli case, we have \begin{align*}
       & d_{TV}(U_x\times V_y \times P \times Q, U'_x\times V'_y \times Q \times P)
	     - d_{TV}(U_x\times P\times Q, U'_x\times Q \times P)\\
	   & - d_{TV}(V_y\times P\times Q, V'_y \times Q \times P)
	    + d_{TV}(P\times Q, Q\times P)\le 0\,.
        \end{align*}
    Summing it over $x$ and $y$ completes the proof.
	\end{proof}

\section{Proof of Monotonicity in \cref{thm:tv-submodular}}\label{app:monotone}
	\begin{proposition}[Monotonicity]\label{prop:monotone}
    Under the assumption of \cref{thm:tv-submodular}, the set function $F$ is monotone.
\end{proposition}
	\begin{proof}
    Let $A$ and $B$ be two subsets of $U$ such that $A\subseteq B$. For any $A\subseteq U$, let $P_i^A = \bigtimes_{a\in A} P_i^a$. Using the above notation, we have $P_i^B = P_i^A\times P_i^{B\setminus A}$. 
    By the definition of $F$, we have \begin{align*}
        F(A) ={}& d_{TV}\left(P_1^A\times P_0^A,P_0^A\times  P_1^A \right)\\
	={}& \frac{1}{2}\sum_{x,y\in V_A} \left| P_1^A(x)P_0^A(y) - P_0^A(x)P_1^A(y) \right| \\
	={}& \frac{1}{2}\sum_{x,y\in V_A} \left|\sum_{z,w\in V_{B\setminus A}} P_1^A(x)P_0^A(y)P_1^{B\setminus A}(z)P_0^{B\setminus A}(w) - \sum_{z,w\in V_{B\setminus A}}P_0^A(x)P_1^A(y)P_0^{B\setminus A}(z)P_1^{B\setminus A}(w) \right|\\
	\le{}& \frac{1}{2}\sum_{x,y\in V_A}\sum_{z,w\in V_{B\setminus A}} \left| P_1^A(x)P_0^A(y)P_1^{B\setminus A}(z)P_0^{B\setminus A}(w) - P_0^A(x)P_1^A(y)P_0^{B\setminus A}(z)P_1^{B\setminus A}(w) \right|\\
	={}& d_{TV}(P_1^A\times P_1^{B\setminus A}\times P_0^A \times P_0^{B\setminus A}, P_0^A \times P_0^{B\setminus A}\times P_1^A \times P_1^{B\setminus A})\\
	={}& d_{TV}\left(P_1^B\times P_0^B,P_0^B\times  P_1^B \right)\\
	={}& F(B)
	\,.
    \end{align*}
    where the third equality is because \[\sum_{z,w\in V_{B\setminus A}} P_1^{B\setminus A}(z)P_0^{B\setminus A}(w) = \sum_{z,w\in V_{B\setminus A}} P_0^{B\setminus A}(z)P_1^{B\setminus A}(w) = 1\] and the inequality is a consequence of the triangle inequality.
\end{proof}

\section{Proof of Submodularity in \cref{thm:tv-submodular}}\label{app:submodular}
\begin{proposition}[Submodularity]\label{prop:submodular}
    Under the assumption of \cref{thm:tv-submodular}, the set function $F$ is submodular.
\end{proposition}

\begin{proof}
    To show that $F$ is submodular, we need to check its definition that for any $ A\subseteq U 
	$ and $ a,b\in U\setminus A $ such that $ a\ne b $, it holds that \[ 
	F(A\cup\{a\})+F(A\cup\{b\}) \ge F(A\cup\{a,b\})+F(A)\,,
	 \]
	 If we define $P_i^A = \bigtimes_{a\in A} P_i^a$, the above definition  is equivalent to 
	\begin{align*}
	    & d_{TV}\left(  P_1^a\times P_0^a \times P_1^A\times P_0^A, P_0^a\times P_1^a \times P_0^A\times P_1^A \right)\\
	   & + d_{TV}\left( P_1^b \times P_0^b \times P_1^A\times P_0^A,
	    P_0^b \times P_1^b \times P_0^A\times P_1^A\right)\\
	    \ge{}&
	    d_{TV}\left(  P_1^a\times P_0^a \times P_1^b \times P_0^b \times P_1^A\times P_0^A,
	    P_0^a\times P_1^a \times P_0^b \times P_1^b \times P_0^A\times P_0^A\right)\\
	   & + d_{TV}\left( P_1^A\times P_0^A, P_0^A\times P_1^A \right)\,.
	\end{align*}
	Re-arranging the terms yields 
	\begin{align*}
	  &\  d_{TV}\left(  P_1^a\times P_0^a \times P_1^b \times P_0^b \times P_1^A\times P_0^A,
	    P_0^a\times P_1^a \times P_0^b \times P_1^b \times P_0^A\times P_0^A\right)\\
	   -{}&\ d_{TV}\left(  P_1^a\times P_0^a \times P_1^A\times P_0^A, P_0^a\times P_1^a \times P_0^A\times P_1^A \right)\\
	    -{}&\ d_{TV}\left( P_1^b \times P_0^b \times P_1^A\times P_0^A,
	    P_0^b \times P_1^b \times P_0^A\times P_1^A\right)\\
	    +{}&\ d_{TV}\left( P_1^A\times P_0^A, P_0^A\times P_1^A \right)\le 0\,.
	\end{align*}
	The above inequality follows from \cref{lem:general-case} if we set $P = P_1^A$, $Q = P_0^A$, $R = P_1^a\times P_0^a$, $R' = P_0^a\times P_1^a$, $S = P_1^b\times P_0^b$, and $S' = P_0^b\times P_1^b$. Note that $R\sim R'$ and $S\sim S'$ via the transpose involution (see \cref{rmk:transpose}).
\end{proof}

\end{appendices}

\end{document}